\RequirePackage{etoolbox}

\PassOptionsToPackage{hypertexnames=false}{hyperref}
\documentclass[12pt,final]{colt2021}  %
\usepackage{amsthm}

\newcommand{\thetafinal}{\theta^+}
 
 \newcommand{\simb}{\beta_{\textrm{sim}}}  %
 \newcommand{\sol}{\text{Sol}}
 \newcommand{\concat}[1]{\overline{#1}}
 \newcommand{\ld}{\overline{\Delta}}
 \newcommand{\hd}{\hat{\Delta}}
 \newcommand{\taufinal}{\tau^+}
 \newcommand{\tpp}{\pi_{\text{TP}}}
 \newcommand{\tensorplan}{{\sc TensorPlan}\xspace}
 
 \newcommand{\approxmeasure}{{\sc ApproxTD}\xspace}
 \newcommand{\dime}{\text{dim}_E}
 
 \newcommand{\Probab}{\mathbb{P}}
 
 \newcommand{\ordot}{\tilde{\mathcal{O}}}
 
\newcommand{\bfitDelta}{X}
\newcommand{\effectivehorizon}{{H_{\gamma,\delta}}}
\newcommand{\effectivehorizonshort}{H}
\newcommand{\flattheta}{{\bar{\theta}}}
 \usepackage{algpseudocode}
\usepackage{algorithmicx}
\usepackage{varwidth}

\algnewcommand{\algorithmicgoto}{\textbf{goto}}%
\algnewcommand{\Goto}[1]{\algorithmicgoto~\ref{#1}}%
\algnewcommand{\Break}{\textbf{break}}%

\algnewcommand{\Initialize}[1]{%
  \State \textbf{Initialize:}
  \Statex \hspace*{\algorithmicindent}\parbox[t]{.8\linewidth}{\raggedright #1}
}
\algnewcommand{\Inputs}[1]{%
  \State \textbf{Inputs:}
  \Statex \hspace*{\algorithmicindent}\parbox[t]{.8\linewidth}{\raggedright #1}
}

\usepackage[capitalize]{cleveref}

\usepackage[utf8]{inputenc}
\usepackage[english]{babel}
\usepackage{latexsym}
\usepackage{mathrsfs}
\usepackage{mathtools}
\usepackage{bm}
\usepackage{datetime}
\usepackage{tikz}
\usepackage{listings}
\usepackage{mdframed}
\usepackage{pgfplots}
\usepackage{pgfplotstable}
\usepackage{dsfont}
\usepackage{color}
\usepackage{colortbl}
\usepackage{pifont}
\usepackage[bf,font=small,singlelinecheck=off]{caption}

\usepackage{microtype} %
\usepackage{float}
\usepackage{xfrac} %
\usepackage{xspace}

\renewcommand{\phi}{\varphi}

\usepackage{nicefrac}
\usepackage{wrapfig}

\usepackage[normalem]{ulem}

\newcommand{\X}{\mathcal{X}}

\newcommand{\sX}{\Sigma}

\DeclareMathOperator{\Dists}{\mathcal{M}_1}

\newcommand{\bN}{\mathbb{N}}
\newcommand{\bR}{\mathbb{R}}

\DeclareMathOperator{\poly}{poly}

\definecolor{emerald}{rgb}{0.31, 0.78, 0.47}

\usepackage{enumerate}
\newtheorem{theorem}{Theorem}[section]

 \newtheorem{corollary}[theorem]{Corollary}
 \newtheorem{claim}[theorem]{Claim}
 \newtheorem{definition}[theorem]{Definition}
\newtheorem{assumption}[theorem]{Assumption}
 \newtheorem{remark}[theorem]{Remark}

\usepackage{aliascnt}

\usepackage{xspace}

\usepackage{enumerate}

\newcommand{\E}{\mathbb E}

\newcommand{\ip}[1]{\left\langle #1 \right\rangle}

\newcommand{\norm}[1]{\|#1\|}
\newcommand{\R}{\mathbb{R}}
\newcommand{\N}{\mathbb{N}}
\newcommand{\cA}{\mathcal{A}}
\newcommand{\cB}{\mathcal{B}}
\newcommand{\cBtheta}{B}

\newcommand{\cE}{\mathcal{E}}
\newcommand{\cF}{\mathcal{F}}

\newcommand{\cM}{\mathcal{M}}

\newcommand{\cO}{\mathcal{O}}

\newcommand{\cS}{\mathcal{S}}

\newcommand{\sM}{\mathscr M}

\newcommand{\one}[1]{\mathbb{I}\{#1\}}

\DeclareMathOperator{\supp}{supp}

\newcommand{\simulatesc}{{\textsc{Simulate}}}
\renewcommand{\epsilon}{\varepsilon}
\newcommand{\ceil}[1]{\left\lceil {#1} \right\rceil}
\newcommand{\floor}[1]{\left\lfloor {#1} \right\rfloor}

\DeclareMathOperator*{\argmin}{arg\ min}
\DeclareMathOperator*{\argmax}{arg\ max}

\newcommand{\bbP}{\mathbb{P}}

\DeclareMathOperator{\vectorize}{vectorize}

\newif\ifsup\suptrue

\usepackage{algorithm}
\usepackage{algpseudocode}

\usepackage{times}
\usepackage{newtxmath}

\usepackage{thmtools}
\usepackage{thm-restate}

\title[
Query-efficient Planning under
Linear Realizability of the Optimal Value Function
]
{
On Query-efficient Planning in MDPs under
Linear Realizability of the Optimal State-value Function
}
\coltauthor{%
 \Name{Gell\'ert Weisz}\\
 \addr{DeepMind, London, UK}\\
 \addr{University College London, London, UK}
 \AND
 \Name{Philip Amortila}\\
 \addr{University of Illinois, Urbana-Champaign, USA}
 \AND
 \Name{Barnab\'as Janzer}\\
 \addr{University of Cambridge, Cambridge, UK}
 \AND
 \Name{Yasin Abbasi-Yadkori}\\
 \addr{DeepMind, London, UK}
 \AND
 \Name{Nan Jiang}\\
 \addr{University of Illinois, Urbana-Champaign, USA}
 \AND
 \Name{Csaba Szepesv\'ari}\\
 \addr{DeepMind, London, UK}\\
 \addr{University of Alberta, Edmonton, Canada}
}

\begin{document}

    \maketitle

\begin{abstract}

We consider the problem of local planning in fixed-horizon Markov Decision Processes (MDPs)
with a generative model
under the assumption that the optimal
value function lies close to the span of a feature
map. The generative model provides a restricted, ``local'' access to the MDP: The planner can ask for random transitions from previously returned states and arbitrary actions, and the features
are also only accessible for the states that are encountered in this process.
As opposed to previous work (e.g. \cite{LaSzeGe19}) where linear realizability of \textit{all} policies was assumed, we consider the significantly relaxed assumption of a single linearly realizable (deterministic) policy.
A recent lower bound by \citet{weisz2020exponential} established that the related problem when the action-value function of the optimal policy is linearly realizable
requires an exponential number of queries, either in $H$ (the horizon of the MDP) or $d$ (the dimension of the feature mapping).
Their construction crucially relies on having an exponentially large action set.
In contrast, in this work, we establish that $\poly(H,d)$ planning \textit{is} possible with state value function realizability whenever the action set has a constant size.
In particular, we present the \tensorplan algorithm which uses $\poly((dH/\delta)^A)$
simulator queries to find a $\delta$-optimal policy relative to \emph{any} deterministic policy for which the value function is linearly realizable with some bounded parameter (with a known bound).

This is the first algorithm to give a polynomial query complexity guarantee using only linear-realizability of a single competing value function.
Whether the computation cost is similarly bounded remains an interesting open question.
We also extend the upper bound  to the near-realizable case and to the infinite-horizon discounted MDP setup.
The upper bounds are complemented by a lower bound which states that
in the infinite-horizon episodic setting,
planners that achieve constant suboptimality need exponentially many queries, either in the dimension or the number of actions.

\end{abstract}

\section{Introduction}

We are concerned with the problem of \emph{planning} in large Markov Decision Processes (MDPs) using a \emph{simulator} (or generative model),
with a query complexity---that is, the number of  calls to the simulator---independent of the size of the state space. While such a result is possible by running a Monte-Carlo tree search algorithm if we are concerned with finding out a good action with some computation every time a state is encountered,
these methods require \emph{exponential} in the planning horizon $H$ number of queries,
which, in the worst-case and without further assumptions, is unavoidable
\citep{kearns2002sparse}.

In the hope of avoiding such exponential dependence,
we consider the setting when
the simulator gives the planner
access to a feature mapping that maps states to $d$-dimensional vectors.
The idea is to use the linear combination of features with some fixed parameter vector to approximate  value functions in the MDP \citep[e.g.][]{SchSei85}.
A basic question, then, is when and how such (linear) function approximation schemes enable query-efficient learning. One minimal assumption, which we consider here, is that the optimal value function $v^\star$ can be represented as a linear function of the feature mapping and an unknown $d$-dimension parameter. Finding this $d$-dimensional coefficient would then grant access to $v^\star$,
 and choosing a near-optimal action for a given state is then possible using low-cost
 one-step lookahead planning.

Despite that the number of unknowns is substantially reduced to $d$, it is not clear at all whether this setting is tractable. In fact, in a closely related setting where the optimal \textit{state-action} value function $q^\star$ is linear, \citet{weisz2020exponential} have recently shown that the query complexity is still exponential in $\Omega(\min(H,d))$. Crucially, their construction relies on having exponentially many actions, leaving open the possibility that a small action set will enable polynomial query complexity.

In our setting, where $v^\star$ rather than $q^\star$ is realizable, finding a near-optimal action amongst $A$ actions trivially requires $\Omega(A)$ queries, even when $v^\star$ is known.
Thus, in order to enable polynomial-time learning, we consider the setting of \textit{small actions sets}. To summarize, the central question we address in this paper is the following:

\begin{center}
\emph{ Is a polynomial query complexity achievable under linear realizability of $v^\star$, when the number of actions is $A=\cO(1)$?}
\end{center}

We provide a positive result to this question in the \emph{fixed-horizon} setting, where our algorithm \tensorplan enjoys a per-call query complexity $\poly((dH/\delta)^A)$,
where %
$H$ is the horizon and $\delta$ is the suboptimality target that the policy induced by continuously running the planning algorithm at every state encountered needs to satisfy. %
Given an input state at the beginning of the horizon,
in its initialization phase, \tensorplan uses simulations to estimate the parameters of $v^\star$.
In this and subsequent calls, given an input state, the estimated $v^\star$ is used by another procedure that uses additional simulations to compute one-step lookahead action-value estimates.
We prove that the resulting policy loses at most $\delta$ total expected reward compared to optimality, regardless of the choice of the initial state, while the number of queries both for the initialization and the subsequent steps stays below the quoted polynomial bound.

In fact, \tensorplan works in a more general setting -- it will automatically compete with the \textit{best deterministic policy} whose value function is realizable by the features. This recovers the previously mentioned ``classic'' setting: when $v^\star$ is realizable the best deterministic policy is an optimal policy $\pi^\star$.
Loosely, the initialization phase of our algorithm works in the following way:
The algorithm keeps track of list of critical data that is used to refine a hypothesis set that
contains
those $d$-dimensional parameter vectors
that (may) induce a value function for some deterministic policy. Call these parameter vectors consistent.
The algorithm refines its hypothesis set in a number of phases.
For this, at the beginning of a phase, it chooses a parameter vector from the hypothesis set
that maximizes the total predicted value at the initial state; an ``optimistic choice''.
Next, the algorithm runs a fixed number of tests to verify that the parameter vector chosen gives a value function of some policy. If this consistency is satisfied, it also follows that the predicted value is almost as high as the actual value of the parameter-induced policy. As such, by its optimistic choice, the parameter vector gives rise to the policy whose value function is linearly realizable and whose value is the highest in the initial state.
When the test fails, the hypothesis set is shrunk by
expanding the list of critical data with data from the failed test.
To show that the hypothesis set shrinks rapidly, we introduce
a novel tensorization device which lifts the consistency checking problem to a $d^A$-dimensional Euclidean space where the tests become linear.
This tensorization device allows us to prove that at most $\cO(d^A)$ constraints can be added (in the noise-free case) if there exist a deterministic policy with linearly realizable value function.
We note that with minor modifications to the inputs and analysis, the query complexity guarantees of \tensorplan translate to the discounted MDP setting.

To complement the query complexity upper bound of \tensorplan, we show that a lower bound of $\Omega\left(2^{\min\{d,A/2\}}/d\right)$ applies for any planner with a constant suboptimality, albeit this lower bound is available only for the \emph{infinite-horizon} episodic setting with a total cost criterion. %
The hard examples in the lower bound are (navigation) MDPs
with deterministic dynamics and costs, and $\Theta(d)$ diameter and $\Theta(d)$ actions.
Thus the three differences between the lower bound's setting and that of \tensorplan's upper bound are that the lower bound is in the \emph{undiscounted infinite-horizon}, total \emph{cost} setting and the number of actions \emph{grows with $d$}, albeit only linearly. 

The rest of the paper is structured as follows. The upcoming section (\cref{sec:prelim}) introduces notations, definitions, and the formal problem definition. \cref{sec:exp-infinite-horizon} gives the exponential lower bound for the infinite-horizon setting. \cref{sec:upper} presents the \tensorplan algorithm for efficient planning in the finite-horizon setting, and states the query complexity guarantee (\cref{thm:main}),
as well as an extension of this result to the near-realizable case (\cref{thm:misspecification}) and infinite-horizon discounted case (\cref{thm:discounted}).
\cref{sec:related} discusses related work and we conclude in \cref{sec:conclusion}.

\section{Preliminaries}\label{sec:prelim}

\newcommand{\cSh}{\cS}

We write $\N_+ = \{1, 2, \dots \}$ for the set of positive integers,
$\bR$ for the set of real numbers, and for $i \in \N_+$, $[i] = \{1,\dots,i\}$ for the set of integers from $1$ to $i$. Given a measurable space $(\X,\sX)$, we write $\Dists(\X)$ for the set of probability measures on that space (the $\sigma$-algebra will be understood from context). We write $\supp(\mu)$ for the support of a distribution $\mu$. 

We recall the notation and the most important facts about Markov Decision Processes (MDPs). For further details (and proofs) the reader is referred to the book of \citet{Put94}. 
In this paper an MDP is given by a tuple $\sM=(\cS,\sX,\cA,Q)$, where %
$(\cS, \sX)$ is a measurable state space, $\cA$ is a set of actions and for each $(s,a)\in\cS\times\cA$, $Q_{s,a} \in \Dists([0,1] \times \cS)$ is a probability measure on rewards and next-state transitions received upon taking action $a$ at state $s$. Note that it follows that the random rewards are bounded in $[0,1]$.
We denote by $r_{sa}$ the expected reward when using action $a$ in state $s$: $r_{sa} = \int r Q_{sa}(dr,ds')$. 
Further, we let  $P_{sa}\in \Dists(\cS)$ denote the distribution of the next-state: $P_{sa}(ds') = \int_{r\in\bR} Q_{sa}(dr,ds')$.
We assume that $\cA$ is finite, and thus without loss of generality we let $\cA=[A]$ for some integer $A \geq 2$.

In the \textbf{fixed-horizon setting} with horizon $H\ge1$ the agent (a decision maker) %
interacts with the MDP in an $H$-step sequential process as follows:
The process is initialized at a random initial state $S_1\in \cS$. %
In step $h\in [H]$, the agent first observes the current state $S_h\in\cS$, then chooses an action $A_h\in\cA$ based on the information available to it. 
The MDP then gives a reward $R_h$ and transitions to a next-state $S_{h+1}$, where $(R_{h}, S_{h+1})\sim Q_{S_h, A_h}$.
After time-step $H$, the episode terminates.

The goal of the agent is to maximize the total expected reward $\sum_{h\in[H]} R_h$ for the episode by choosing the actions based on the observed past states and actions in the episode. A \textit{(memoryless) policy} $\pi$ takes the form $(\pi^{(h)})_{h\in[H]}$ where $\forall h\in[H]$, $\pi^{(h)}: \cSh \rightarrow \Dists(\cA)$.
A \textit{deterministic policy} $\pi$ further satisfies that for any $h\in[H]$ and $s\in\cSh$ there exists $a\in\cA$ such that $\pi^{(h)}(s)=\delta_a$ where $\delta$ is the Dirac delta distribution.
Given a memoryless policy $\pi$, a state $s\in \cS$ and step $h\in [H]$ within an episode,
the value $v_h^\pi(s)$  is defined as the total expected reward incurred until the end of the episode when the MDP is started from $s$ in step $h$ and $\pi$ is followed throughout. Writing
$\mu f = \int f(s') \mu(ds')$
for the expected value of a measurable function $f:\cS \to \R$ with respect to $\mu \in \Dists(\cS)$, these values are known to satisfy
\begin{align*}
v_h^\pi(s) = r_\pi(s) + P_\pi(s) v_{h+1}^\pi\,, \qquad s\in \cS\,,
\end{align*}
where $v_{H+1}^\pi =0$, $r_\pi(s) = \sum_{a\in \cA} \pi(a|s)r_{sa}$, and $P_\pi(s)(ds') = \sum_{a\in \cA} \pi(a|s) P_{sa}(ds')$. The maximum value achievable from a state $s\in \cS$ when in step $h\in [H]$ is denoted by $v_h^\star(s)$. We also define $v_{H+1}^\star(s)=0$, for convenience. We let $v^\star = (v^\star_h)_{h\in [H+1]}$ and call $v^\star$ the optimal value function. It is known that $v^\star$ satisfies the recursive \textit{Bellman optimality equations}:
\begin{align}
v_h^\star(s)&=\max_{a\in\cA} \left\{ r_{sa} + P_{sa} v_{h+1}^\star \right\}\,, \qquad s\in \cS\,.
\label{eq:bopt}
\end{align}
As is well known, the policy that in state $s\in \cS$ chooses an action that maximizes the right-hand side of \cref{eq:bopt}, is optimal. It also follows that there is always at least one optimal deterministic memoryless policy.

By taking $H\to\infty$, we obtain the \textbf{infinite-horizon total reward setting}.
Here, policies are described by infinite sequences of probability kernels mapping histories to actions.
The interconnection of a policy and an MDP puts a probability distribution over the space of infinitely long histories and the value of a policy $\pi$ in state $s\in \cS$ is the total expected reward incurred when the policy is started from state $s$. 
Note that the total expected reward may be undefined, or take on $\bR \cup\{ \pm\infty\}$.
A policy is \textbf{admissible} if this value $v^\pi(s)$ is well-defined and not $-\infty$ no matter the initial state (we do not mind policies that generate infinite reward, but mind policies which generate infinite cost). In what follows we only consider MDPs in the infinite horizon setting where there exist at least one admissible policy.
The optimal value $v^\star(s)$ of a state $s\in \cS$, similarly to the finite-horizon case, is defined as the largest value that can be obtained by some admissible policy.

\subsection{Featurized MDPs, feature map compatible optimal values}

As noted earlier, we provide the planner with a feature mapping which captures the optimal value function. In the finite-horizon setting this translates to the existence of some $\theta^\star$ such that
\begin{align}
v^\star_h(s) = \ip{\phi_h(s),\theta^\star}\,, \qquad \text{for all } h\in [H] \text{ and } s\in \cS\,.
\label{eq:vstr}
\end{align}
We also consider the nearly-realizable case, where for some ``misspecification'' parameter $\eta\ge 0$ there exists some $\theta^\star$ such that
\begin{align}
\left|v^\star_h(s) - \ip{\phi_h(s),\theta^\star}\right|\le \eta\,, \qquad \text{for all } h\in [H] \text{ and } s\in \cS\,.
\label{eq:vstr-approx}
\end{align}
The parameter $\theta^\star$ is unknown to the planner in both cases. Here, $\phi_h: \cS \to \R^d$ is the so-called feature map. 
As will be described in more details in the next section,
the planner is given \emph{local access} to the feature map. That is, the planner can access $\phi_h(s)$ for all the states $s\in \cS$ that it has previously encountered while interacting with the simulator, but has no access to the features of other states.
For convenience, in the finite horizon-setting we will also define $\phi_{H+1}(s) = \boldsymbol{0}$ for all $s \in \cS$, regardless of the other maps. 
An MDP together with a feature map $\phi = (\phi_h)_{h\in [H]}$ on its state-space is called a \textit{\textbf{featurized MDP}}.
When \cref{eq:vstr} holds we say that \textit{\textbf{$v^\star$ is (linearly) realizable by the feature map $\phi$}}.

In this paper we consider a setting that relaxes linear realizability of the optimal value function. 
To define this setting we need the notion of \textit{\textbf{$v$-linearly realizable policies}}:
\begin{definition}[$v$-linearly realizable policies] \label{def:v-realizable-policies}
We say that a policy $\pi$ is $v$-linearly realizable
with misspecification $\eta\ge 0$
under the feature map $\phi = (\phi_h)_{h\in [H]}$
if there exists some $\theta\in \R^d$ such that
its value function satisfies $\left|v^\pi_h(s) - \ip{\phi_h(s),\theta}\right|\le \eta$ for all $h\in [H]$ and $s\in \cS$.
Furthermore, if $\theta$  satisfies $\norm{\theta}_2\le B$ we say that $\pi$ is $B$-boundedly $v$-linearly realizable with misspecification $\eta$ under $\phi$.
\end{definition}
In what follows we will be concerned with designing a planning algorithm that, given local access to a feature map, competes with the best $v$-linearly realizable memoryless deterministic (MLD) policy under that feature map (if one exists) in the following sense:
For $B>0$ and $\eta\ge0$, define the function $v_{B,\eta}^\circ:\cS \to \R$ as
\begin{align}
\begin{split}\label{eq:phi-compatible-value}
v_{B,\eta}^\circ(s) = \sup\big\{ v^\pi_1(s) \,:\, &\pi \text{ is MLD and is } B\text{-boundedly 
$v$-linearly realizable}\\ &\text{with misspecification }\eta\text{ given } \phi \big\}\,.
\end{split}
\end{align}
We call will $v_{B,\eta}^\circ$ the \textit{\textbf{$\phi$-compatible optimal value function} at scale $B$ and misspecification $\eta$}.
Note that if there are no $v$-linearly realizable policies with misspecification $\eta$ in an MDP, $v_{B,\eta}^\circ(s)\equiv -\infty$ for each state $s\in \cS$ of the MDP. Competing with the best $v$-linearly realizable MLD policy (at scale $B>0$ and misspecification $\eta\ge0$) means the ability to generate actions of a policy whose value function is close to $v_{B,\eta}^\circ$ (for the fully formal definition, see the next section).
Note that if the optimal value function of an MDP is linearly realizable with parameter vector $\theta^\star$ and misspecification $\eta'$
then for any $B\ge \norm{\theta^\star}_2$, 
$v_{B,\eta}^\circ = v^\star$ for any $\eta\ge \eta'$.
Hence, the setting we introduce generalizes the one where the optimal value function is exactly or near-realizable with $B$-bounded parameter vectors.

In the \textbf{infinite-horizon setting},
the constraint  equivalent to linear realizability of the optimal value function is 
that with some $\theta^\star\in \R^d$, again unknown to the planner, 
\begin{align*}
\left|v^\star(s) - \ip{\phi(s),\theta^\star}\right|\le\eta\,, \qquad \text{for all } s\in \cS\,,
\end{align*}
where now the feature map is $\phi: \cS \to \R^d$, i.e., with no dependence on the step within the episode.
Featurized MDPs then are defined by joining an MDP with such a (homogeneous) feature map. 
The other notions introduced above have also natural counterparts. To save space, these are not introduced explicitly, trusting that the reader can work out the necessary modifications to the definitions without introducing any ambiguity.

\subsection{Local Planning} %

In the \textit{\textbf{fixed-horizon local planning}} problem, 
a \textit{\textbf{planner}} is given an input state and is tasked with computing a near-optimal action for that state while interacting with a black-box that simulates the MDP.
In the \textit{\textbf{(state-)featurized local planning}} problem, the black-box also returns the feature-vector of the next state. Access to the black-box is provided by means of calling a function $\simulatesc$, whose semantics is essentially as just described, but will be further elaborated on below.

More formally, the planner needs to ``implement'' a function, 
which we call {\tt GetAction} and whose semantics, in the context of fixed-horizon MDPs, is as follows:

\begin{definition}[{\tt GetAction}$(d,A,H,\simulatesc,s,h,\phi_h(s),\delta, B)$]
The meaning of inputs is as follows:
$d$ is the dimension of the underlying feature map, 
 $A$ is the number of actions, 
 $H$ is the episode length,
$\simulatesc$ is a function that provides access to the oracle that simulates the MDP, %
$s$ is the state where an action is needed at stage $h\in [H]$,
 $\delta>0$ is a suboptimality target,
 and $B$ is the parameter vector bound.
This function needs to return an action in $ \cA$ with the intent that this is a ``good action'' to be used at stage $h$ when the state is $s$. 
\end{definition}

Given a featurized MDP and a planner as described above, 
the planner \textit{\textbf{induces a (randomized, possibly memoryful) policy}},
which is the policy that results from calling {\tt GetAction} along a trajectory and following its recommended actions.
If the initial state is $S_1=s_0 \in \cS$, the first action taken by this policy is 
$A_1 = {\tt GetAction}(\dots,S_1,1,\dots)$, 
the second is 
$A_2 = {\tt GetAction}(\dots,S_2,2,\dots)$ where $S_2 \sim P_{S_1,A_1}$, etc.
If {\tt GetAction} does not save data between the calls, the resulting policy would be memoryless, but this is not a requirement. 
\emph{In fact, we require that {\tt GetAction} is first called with $h=1$ and then $h=2$, etc.}
A practical planner which is used across multiple episodes can also save data between episodes. In this case {\tt GetAction} can be called with $h=1$ after being called with $h=H$, designating the start of a new episode. 
For now, we assume that this is not the case, as this allows for cleaner definitions.%
\footnote{Jumping a bit ahead of ourselves,
if we cared about long-run average per-state query-complexity, 
one could perhaps do better by allowing planners to save data between episodes.
}

Inside {\tt GetAction} the planner can issue any number of calls to $\simulatesc$.
The function $\simulatesc$ takes as inputs a state-stage-action triplet $(s,h,a)$. In response, $\simulatesc$ 
returns a triplet $(R, S', \phi_{h+1}(S'))$ where $(R, S')$ is a ``fresh'' random draw from $Q(s,a)$. For generality the simulator is also allowed some inaccuracy, in the sense that it returns $([R+\Lambda_{sa}]_0^1,S',\phi_{h+1}(S'))$ where $\Lambda_{sa}\in\R$ is a constant satisfying $|\Lambda_{sa}|\le\lambda$, for some $\lambda\ge 0$ that we call the simulator's accuracy, and $[x]_0^1=\max(0,\min(1,x))$ (ie. inaccurate rewards are clipped in $[0,1]$). 
Neither $\Lambda_{sa}$ nor $\lambda$ are known to the planner.
The planner can only access states that it is given access to either when {\tt GetAction} is called, or returned by a call to $\simulatesc$. The same holds for the features of the states.
We note that this is essentially the same setting as what is called \emph{sampling with state revisiting} by \citet{li2021sample}.

The \textit{\textbf{quality of a planner}} is, on one hand, assessed based on the quality of the policy that it induces and, on the other hand, by its \textit{\textbf{worst-case (per-episode) query-cost}}, which is defined as the largest total query-cost (ie. number of calls to $\simulatesc$ made by {\tt GetAction}) 
encountered while running the planner for the $H$ stages of an episode, starting at stage $h=1$.
Our lower bounds in \cref{sec:exp-infinite-horizon} will be given in terms of 
the \textit{\textbf{worst-case (per-state) query-cost}},
which is defined as the largest query-cost encountered during any single call to {\tt GetAction}. 
Note that the switch from per-episode to per-state cost only strengthens the lower bound as the worst-case per-episode query-cost is at least as large as the worst-case per-state query-cost.

\begin{definition}[Sound planner]\label{def:sound}
Let $B,\delta>0$, $\lambda,\,\eta\ge0$, $H\ge1$.
A planner is $(\delta,B)$-\textbf{sound} with simulator accuracy $\lambda$ and misspecification $\eta$
if for any featurized $H$-horizon MDP $(\sM,\phi)$ with rewards bounded in $[0,1]$ and with $1$-bounded feature maps (i.e.\ for all $h\in[H]$, $s'\in S$, $\norm{\phi_h(s')}_2\le 1$),
the (random) $H$-horizon policy $\pi$ that the planner induces while interacting with the $\lambda$-accurate simulation oracle satisfies
\begin{align}\label{eq:soundness-satisfies-gap}
v^\pi_1(s) \ge v^\circ_1(s)-\delta\quad \text{for all $s \in \cS$}\,, 
\end{align}
where $v^\pi$ is the $H$-horizon value function of $\pi$ in $\sM$ and $v^\circ=v^\circ_{B,\eta}$ is the  $H$-horizon $\phi$-compatible optimal value function of $\sM$ (cf. Equation \eqref{eq:phi-compatible-value}).
\end{definition}
The reader interested in exploring alternatives to the protocol described here is referred to 
\cref{apx:inter}, where we also discuss ``nuances'' like how to work with ``arbitrary'' state spaces.

\paragraph{Further notations}For $v\in\bR^d$, and $a\le b\le d$ positive integers, let $v_{a:b}\in \bR^{b-a+1}$ be the vector corresponding to the entries with indices in $\{a,a+1,\dots,b\}$, i.e., $(v_{a:b})_i = v_{a+i-1}$.
For $x\in\bR, v\in\bR^d$, let $\concat{xv}\in\bR^{d+1}$ denote the concatenation of $x$ and $v$ such that $(\concat{xv})_1=x$ and $(\concat{xv})_{2:d+1}=v$. 
We write $\otimes: \bR^{d_1} \times \bR^{d_2} \mapsto \bR^{d_1 \times d_2}$ for the \textit{tensor product} of two vectors, defined as 
$(u \otimes v)_{i \in [d_1], j \in [d_2]} = u_i \cdot v_j %
$
For a larger set of vectors $(u^{(1)}, \dots u^{(n)})$, $n \in \bN_+$, we write  $$
\left(\otimes_{i\in[n]}u^{(i)}\right)_{j_1,j_2,\ldots,j_n}=\prod_{i\in[n]} (u^{(i)})_{j_i} \quad \quad u^{(i)} \in \bR^{d_i}.
$$ We will frequently view tensors in $\bR^{\times_{i \in [n]} d_i}$ as vectors in $\bR^{\prod_{i \in [n]} d_i}$ via the usual isomorphism. Notationally, this is given by the \textit{vectorize} operation. Letting $\flat: [d_1] \times \dots \times [d_n] \mapsto [\prod d_i]$ be any bijection between indices, we define:
$\left(\vectorize(T)\right)_{\flat(j_1,\dots,j_n)} = T_{j_1,j_2,\dots,j_n}$, for $T \in \bR^{\times_{i \in [n]} d_i}$.
Let us define the inner product of two compatible tensors to be
\[
\ip{\otimes_{i \in [n]} u^{(i)}, \otimes_{i \in [n]} v^{(i)}} =
\ip{\vectorize(\otimes_{i \in [n]} u^{(i)}), \vectorize(\otimes_{i \in [n]} v^{(i)})} \,.
\]
The key property which we need is that inner product between the two tensors %
is then seen to be: %
\begin{equation}\label{eq:flat-ip}
\ip{\otimes_{i \in [n]} u^{(i)}, \otimes_{i \in [n]} v^{(i)}} = \prod_{i \in [n]} \ip{ u^{(i)},v^{(i)}} 
\end{equation}

\section{Exponential query complexity for infinite-horizon problems}
\label{sec:exp-infinite-horizon}

In this section, we show that in infinite-horizon episodic problems, the query complexity of $(\frac 1 2,\sqrt{d})$-sound planners is exponential in the dimension of the feature map $d$, even with no misspecification (ie. $\eta=0$):
\begin{theorem}\label{thm:lb}
Fix $d>1$.
For any local planner $P$ whose worst-case per-state query-cost is at most $\simb$ 
there exists a featurized MDP with $d$-dimensional features and
$2d-1$ actions, such that the optimal value function of the MDP
is exactly realizable with the features (Eq.~\ref{eq:vstr})
and takes values in the $[-2,0]$ interval,
the 2-norm of feature vectors is $\cO(1)$, the 2-norm of the parameter vector of the optimal value function is $\cO(\sqrt{d})$
and the suboptimality $\delta$ of the policy induced by $P$ in the MDP satisfies
\begin{align*}
\delta \ge \frac{2^{d-2}}{d( \simb+1)}-1\,.
\end{align*}
In particular, to guarantee $\delta\le 0.5$, the query-cost must satisfy $\simb\ge \Omega\left(2^{\min\{d,A/2\}}/d\right)$. %
\end{theorem}
We provide a proof sketch only. The full proof is given in \cref{apx:lbproof}.
\begin{proof}[Sketch]
Consider an MDP whose state space is $\{-1,0,1\}^d$, with deterministic dynamics:
In every state, there are at most $2d$ actions to increment or decrement one coordinate ($1$ or $-1$ can only be changed to $0$, $0$ can be changed to either $-1$ or $+1$) and there is an additional action that does not change the state.
Instead of rewards, it will be more convenient to consider costs.
The cost everywhere is 1, except for staying put in a special state, $s^\star$, where the cost is zero. Thus, the optimal policy at any state takes the shortest path to $s^\star$, which gives that $v^\star(s)$ is the $\ell^1$ distance between $s$ and $s^\star$.
Restricting $s^\star$ to take values in $\{-1,1\}^d$, setting
$\phi(s) = \concat{d s}$, we have $v^*(s) = \ip{\phi(s),\concat{(-1) s^\star}}$.
A planner that takes less than $2^{d-1}$ queries has no better than $1/2$ chance of discovering $s^\star$ (which is one of $2^d$ possibilities). If a planner does not discover $s^\star$, the actions it takes are uninformed. Hence, for any planner that uses a small query budget, one can hide $s^\star$ and force the planner to wander around in the MDP for a long time. One can then scale costs to finish the proof.
\end{proof}

\section{Efficient planning for the finite-horizon setting}\label{sec:upper}

In this section, we present \tensorplan (\cref{alg:local})  and prove its
soundness (cf. \cref{def:sound}) and efficiency (\cref{thm:main}).
We start with a high-level description of the main ideas underlying the planner.
Initially, we only prove soundness for exact realizability (ie. $\eta=0$), which we later generalize in \cref{thm:misspecification}.

The planner belongs to the family of generate-and-test algorithms.
To describe it,
let $\sM = (\cS,\Sigma,\cA,Q)$ denote the MDP that the planner interacts with
and let $\phi = (\phi_h)_{h\in [H]}$ be the underlying feature map.
Further, let $\Theta^\circ\subset \R^d$ be the set which collects the parameter vectors of the value functions of $B$-boundedly $v$-linearly realizable DML policies with misspecification $\eta=0$ (\cref{def:v-realizable-policies}). That is, $\Theta^\circ$ is such that for any $\theta\in \Theta^\circ$, $\norm{\theta}_2\le \cBtheta$ and for some DML policy $\pi$ of $\sM$,
\begin{align}\label{eq:theta-circ-satisfies}
v_h^\pi(s) = \ip{\phi_h(s),\theta}\, \quad \text{for all $h\in[H]$ and $s\in\cS$}.
\end{align}

Let $s_0$ be the state which the planner is called for. The algorithm will maintain a subset $\Theta$ of $\R^d$
such that, with high probability, $\Theta^\circ\subset \Theta$.
The set is initialized to the $\ell^2$-ball of radius $\cBtheta$, which obviously satisfies this constraint. Given the set $\Theta$ of admissible parameter vectors and $s_0\in \cS$, the planner finds the \textit{optimistic} parameter vector $\theta^+ = \argmax_{\Theta} \ip{\phi_1(s_0),\theta}$ from the set $\Theta$.
Let us write $v_h(s;\theta) \coloneqq \ip{\phi_h(s),\theta}$.
If $\theta\in \Theta^\circ$ then for any $h\in [H]$ and $s\in \cS$,
since the policies defining $\Theta^\circ$ are deterministic,
it follows that there exists an action $a\in \cA$ such that
\begin{align}
v_h(s;\theta) = r_{sa} + P_{sa} v_{h+1}(\cdot;\theta)\,.
\label{eq:loccons}
\end{align}
For any $\theta$, let $\pi_\theta$ denote the policy which chooses the action satisfying the above equation when in state $s$ and stage $h$
(when there is no action that satisfies the \text{consistency condition} \cref{eq:loccons}, the policy can choose any action).

To test whether $\theta^+\in \Theta^\circ$, the algorithm aims to ``roll out'' $\pi_{\theta^+}$.
By this, we mean that upon encountering a state $s$ in stage $h$ in such a rollout, 
the algorithm checks whether there is an action $a$ that satisfies \cref{eq:loccons}.
If such an action is found, it is sent to the simulator, which responds with the next state.
If no such action is found, the test fails -- this means that $\theta^+\not\in \Theta^\circ$.
When this happens, the data corresponding to the transition where the test failed is used to refine the set of admissible parameter vectors and a new admissible set $\Theta'$ is established.
Assuming that the test failed at stage $h^\star$ and state $s^\star$, this new set is
\begin{align*}
\Theta' = \{ \theta \in \Theta \,: \, \exists a\in \cA \text{ s.t. } \cref{eq:loccons} \text{ holds with }
s= s^\star\text{ and } h = h^\star
\}\,.
\end{align*}
Then the testing of $\theta^+$ is abandoned,
$\Theta$ is updated to $\Theta'$, and the process is repeated.
Clearly, $\Theta^\circ\subset \Theta'$ still holds, so $\Theta^\circ \subset \Theta$ also holds after the update.

When a rollout continues up to the end of the episode without failure, the algorithm is given
some evidence that $\theta^+\in \Theta^\circ$, but this evidence is weak.
This is because the states encountered in a rollout are random, and the trajectory generated may just happen to avoid the ``tricky'' states where the consistency test would fail.
Luckily though, if the algorithm keeps testing with multiple rollouts and the tests do not fail for a sufficiently large (but not too large) number of such rollouts,
this can be taken as evidence that $\pi_{\theta^+}$ is indeed a good policy in starting state $s_0$. 
It may happen that $\theta^+$ is still not in $\Theta^\circ$, but the value of $\pi_{\theta^+}$ cannot be low.

This is easy to see, if for the moment we
add a further,
(seemingly) stronger test. %
This test checks whether $v_1(s_0;\theta^+)$ correctly predicts the value of $\pi_{\theta^+}$ in state $s_0$.
To this end, the test simply takes the average sum of rewards along the rollouts. If we detect that $v_1(s_0;\theta^+)$ is not sufficiently close to the measured average value, the test fails. If this strengthened test does not fail either then this is strong evidence that $v^{\pi_{\theta^+}}_1(s_0)$ is as high as $v_1(s_0;\theta^+)$. Now,
since $\Theta^\circ \subset \Theta$ holds throughout the execution of the algorithm,
$v_1(s_0;\theta^+)\ge \max_{\theta\in \Theta^\circ} v_1^{\pi_\theta}(s_0)=v^\circ_{\cBtheta}(s)$ (since we pick $\theta^+$ optimistically),
and hence policy $\pi_{\theta^+}$ can successfully compete with the best $v$-linearly realizable policy in $\sM$ under $\phi$ and at scale $\cBtheta$ (\cref{eq:phi-compatible-value}).

To complete the description of the algorithm, there are three outstanding issues. The first is that due to the randomizing simulation oracle, for any given state $s\in \cS$, one can only check whether \cref{eq:loccons} holds up to some fixed accuracy and only with high probability. Luckily, this does not cause any issues -- when the tests fail, the parameters can be set so that $\Theta^\circ\subset \Theta$ is still maintained.

The second issue is whether the algorithm is efficient.
(So far we have been concerned only with soundness.)
This is addressed by ``tensorizing'' the consistency test.
For $\psi:\cS \to \R^d$, we let $P_{sa}\psi = \int \psi(s') P_{sa}(ds')$.
Using \cref{eq:flat-ip} we then observe
that the existence of an action such that \cref{eq:loccons} holds is equivalent to:
\begin{align*}
0
& = \prod_{a\in \cA} r_{sa}+\ip{ P_{sa}\phi_{h+1}-\phi_h(s), \theta }
 = \prod_{a\in \cA} \ip{ \overline{ r_{sa}\, (P_{sa}\phi_{h+1}-\phi_h(s)) }, \overline{1 \theta} } \\
 &=  \ip{
\otimes_{a\in \cA} \overline{r_{sa}\, (P_{sa}\phi_{h+1}-\phi_h(s)) }, \otimes_{a\in \cA} \overline{1\theta} } \,.
\end{align*}
Now, defining $M_\theta = \otimes_{a\in \cA} \overline{1\theta}$ and $T_s = \otimes_{a\in \cA} \overline{r_{sa}\, (P_{sa}\phi_{h+1}-\phi_h(s)) }$, we see that $\theta\in \Theta^\circ$ is equivalent to that $\ip{T_s,M_\theta}=0$ holds for all $s\in \cS$. Testing a parameter vector at some state is
equivalent to checking whether $M_\theta$ is orthogonal to $T_s$. Clearly, the maximum number of tests that can fail before identifying an element of $\Theta^\circ$ is at most $d^A$, the dimension of $M_\theta$. Since our tests are noisy, we use an argument based on eluder dimensions (which allow imperfect measurements) to complete our efficiency proof \citep{russo2014learning}.

The final issue is really an optimization opportunity. In our proposed algorithm we do not separately test if the value estimates at $s_0$ are close to the empirical return over the rollouts, and instead rely only on the consistency tests. This can be done since, when consistency holds, the expected total reward in an episode is close to the predicted value.
This follows from a telescoping argument. Let $S_1=s_0,A_1,S_2,A_2,\dots,S_H,A_H,S_{H+1}$ be the state-action pairs in a rollout where the tests do not fail, and note that
\begin{align*}
v^{\pi_{\theta^+}}_1(s_0)
&= \E_{\pi_{\theta^+}}\left[ \sum_{h=1}^H r_{S_h,A_h} \right]
= \E_{\pi_{\theta^+}}\left[ \sum_{h=1}^H v_{h}(S_{h};\theta^+)- v_{h+1}(S_{h+1};\theta^+) \right]
 = v_1(s_0;\theta^+)\,,
\end{align*}
where the first equality uses the definition of $v^{\pi_{\theta^+}}$, the second equality uses $r_{S_h,A_h} = v_h(S_h;\theta^+) -P_{S_hA_h} v_{h+1}(\cdot;\theta^+)$, and the last equality uses that $v_{H+1} \equiv 0$.
When measurements are noisy, a similar telescoping argument gives that with high probability,  $v^{\pi_{\theta^+}}_1(s_0) $ is almost as high as $v_1(s_0;\theta^+)$ when consistency tests do not fail for a number of rollouts.

\subsection{The \tensorplan algorithm}

\newcommand{\CleanTest}{{\tt CleanTest}}
\newcommand{\true}{{\tt true}}
\newcommand{\false}{{\tt false}}
\newcommand{\simm}{\simulatesc}

The pseudocode of {\tt GetAction} of \tensorplan is shown in \cref{alg:local}.
\\
\begin{minipage}[t]{0.57\textwidth}
\begin{flushleft}
\begin{algorithm}[H]
\caption{
\tensorplan.{\tt GetAction}
}\label{alg:local}
\begin{algorithmic}[1]
\State \textbf{Inputs:} $d,A, H, \simm, s, h, \phi_h(s), \delta, B$
\If{$h=1$} \Comment{Initialize global $\thetafinal$}
	\State \begin{varwidth}[t]{\linewidth}
      \tensorplan.{\tt Init}$($\par
        \hskip\algorithmicindent$d,A, H, \simm, s,\phi_1(s) ,\delta)$
      \end{varwidth}
\EndIf
\State $\ld_{\cdot}\gets \text{\approxmeasure}(s,h,\phi_h(s),A,n_2,\simm
)$ \label{line2:avg-calc1} \label{line2:avg-calc} %
\State Access $\thetafinal$ saved by \tensorplan.{\tt Init}
\State \Return $\argmin_{a\in[A]} \left|\ip{\ld_{a},\concat{1\thetafinal}}\right|$ \label{line2:action-choice}
\end{algorithmic}
\end{algorithm}
\end{flushleft}
\end{minipage}
\hfill
\begin{minipage}[t]{0.43\textwidth}
\begin{algorithm}[H]
\caption{\approxmeasure}\label{alg:approx-measure}
\begin{algorithmic}[1]
\State \textbf{Inputs:}
$s,h,\phi_h(s),A,n,\simm$
\For{$a=1$ to $A$}
    \For{$l=1$ to $n$}
      \State \begin{varwidth}[t]{\linewidth}
      $(R_l, S'_l, \phi_{h+1}(S'_l))\gets \big($\par
        \hskip\algorithmicindent $\simulatesc(s,h,a)\big)$
      \end{varwidth}
      \State $\tilde \Delta_l\gets \concat{R_l\left(\phi_{h+1}(S'_l)-\phi_h(s)\right)}$\label{line:approx-measure}
    \EndFor
    \State $\Delta_{a}:=\frac{1}{n}\sum_{l\in[n]} \tilde \Delta_l$
\EndFor
\State \Return $(\Delta_{a})_{a\in [A]}$ \label{line:avg-calc-approxmeasure}
\end{algorithmic}
\end{algorithm}
\end{minipage}
\vspace*{.5em}

The main workhorse of \tensorplan is
the initialization routine, \tensorplan.{\tt Init} (\cref{alg:global}),
which generates a global variable $\thetafinal\in \R^d$ that is an estimate for the parameter of the best realizable value function $v^\circ_{\cBtheta}$. Within an episode, this parameter is used by the current and subsequent calls to {\tt GetAction}.
In particular, given $\thetafinal$, {\tt GetAction} approximately implements $\pi_{\thetafinal}$ of the previous section. For this, {\tt GetAction} calls \approxmeasure\footnote{Thusly named since $\ip{\overline{ r_{sa} (P_{sa}\phi_{h+1}-\phi_h(s)) }, \concat{1\theta}}$ corresponds to the ``temporal difference'' error of value function $v_\theta$ at state-action pair $(s,a)$ \citep{sutton1988learning}.} (\cref{alg:approx-measure}), which produces an  estimate of $\overline{ r_{sa} (P_{sa}\phi_{h+1}-\phi_h(s)) }$ for all actions $a \in \cA$.

The {\tt Init} function uses
\begin{align}\label{eq:sol}
\sol\left(\Delta_1, \dots, \Delta_\tau \right) = \left\{ \theta\in\bR^d \,:\, \norm{\theta}_2\le \cBtheta, \forall i\in[\tau]\,: \,\, \left|\ip{\Delta_i,\, \otimes_{a\in[A]}\concat{1\theta}}\right|\le \frac{H^A\epsilon}{2\sqrt{E_d}} \right\} \,.
\end{align}
where $\varepsilon$ is a function of the target suboptimality and $E_d = \ordot\left(d^A A\right)$,
defined in \cref{eq:eddef},
is an upper bound on the
 the eluder dimension of a tensorized clipped-linear function class
 (cf. \cref{eq:tensorizedclass}).
The $\sol(\cdot)$ set stands for the successfully refined sets $\Theta$ of the previous section and its arguments $\Delta_i\in \R^{(d+1)^A}$ correspond to estimates of $\otimes_{a\in \cA}\overline{ r_{sa} (P_{sa}\phi_{h+1}-\phi_h(s)) }$ for the various states $s$ and stages $h$ where the algorithm detects a failure of the consistency test it runs. Estimates of these in {\tt Init} are obtained by calls to  \approxmeasure.

\begin{algorithm}[t]
\caption{\tensorplan.{\tt Init}}\label{alg:global}
\begin{algorithmic}[1]
\State \textbf{Inputs:}
$d,A,H, \simm,s_0,\phi_1(s_0),\delta$
\State $\bfitDelta\gets\{\}$ \Comment{$\bfitDelta$ is a list} \label{line:x-def}
\State Initialize $\zeta, \epsilon,n_1,n_2,n_3$ via equations \eqref{eq:algzetadef}, \eqref{eq:algepsdef}, \eqref{eq:algn1def}, \eqref{eq:algn2def}, \eqref{eq:algn3def}, respectively.
\For{$\tau=1$ to $E_d+2$}
	\State
	Choose any $\theta_\tau \in\argmax_{\theta\in\sol(\bfitDelta)} \ip{\phi_1(s_0), \theta}$
		\Comment{Optimistic choice} \label{line:new-theta} \label{line:new-iter}
	\State $\CleanTest \gets\true$
	\For{$t=1$ to $n_1$}			\Comment{$n_1$ rollouts with $\theta_{\tau}$-induced policy}
		\State  $S_{\tau t1}=s_0$ \Comment{Initialize rollout}
		\For{$j=1$ to $H$} \Comment{Stages in episode}
			\State $\ld_{\tau tj, \cdot} \gets
			\text{\approxmeasure}(
			S_{\tau tj},j,\phi_j(S_{\tau tj}),A,n_2,\simm
			)$ \label{line:avg-calc1}
			\If{
			$\CleanTest$ and
			$\min_{a\in[A]} \left|\ip{\ld_{\tau tja},\concat{1\theta_\tau}}\right|>\tfrac{\delta}{4H}$}\label{line:consistency-test}
			\Comment{Consistency failure?}
				\State $\hd_{\tau tj,\cdot}\gets \text{\approxmeasure}(
				S_{\tau tj},j,\phi_j(S_{\tau tj}),A,n_3,\simm
				)$ \label{line:avg-calc2} \Comment {Refined data}
				\State $\bfitDelta\text{.append}\left(\otimes_{a \in [A]}\hd_{\tau tja}\right)$ \label{line:new-eluder-element} \Comment{Save failure data}
				\State $\CleanTest \gets \false$ \Comment{Not clean anymore}
			\EndIf
            \State $A_{\tau tj}\gets \argmin_{a\in[A]} \left|\ip{\ld_{\tau tja},\concat{1\theta_\tau}}\right|$ \label{line:action-choice} \Comment{Find most consistent action}
			\State $(R_{\tau tj},S_{\tau tj+1},\phi_{j+1}(S_{\tau tj+1})) \gets \simulatesc(S_{\tau tj},j, A_{\tau tj})$
			 \Comment{Roll forward} \label{line:simulate-choice}
		\EndFor
	\EndFor
	\State \textbf{if} $\CleanTest$ \textbf{then} \Break \Comment{Success?} \label{line:cleantest-break}
\EndFor
\State Save into global memory $\thetafinal\gets\theta_\tau$\label{line:return}
\end{algorithmic}
\end{algorithm}

Note that {\tt Init} as described continues to generate rollout data even after a consistency test fails. This is clearly superfluous and in an optimized implementation one could break out of the test loop to generate the next candidate immediately after a failure happens. The only reason the algorithm is described in the way it is done here is because this allows for a cleaner analysis: every policy will have access to data from $n_1$ rollouts, even if the policy fails a consistency test. %

\begin{remark}
The reader might wonder why \tensorplan follows the most consistent action in \cref{line2:action-choice} of {\tt GetAction}, instead of
the best action according to its $\thetafinal$, which would be $\argmax_{a\in[A]}\ip{\ld_{a},\concat{1\thetafinal}}$.
Indeed, a practical implementation might adopt this, together with the same change to \cref{line:action-choice} of {\tt Init}, and a strengthening of the consistency test of {\tt Init}'s \cref{line:consistency-test} to require that \emph{the best} action (according to $\theta_\tau$) be consistent, instead of \emph{any} action. 
This test would fail if $\left|\max_{a\in[A]} \ip{\ld_{\tau tja},\concat{1\theta_\tau}}\right|>\tfrac{\delta}{4H}$.
One might hope that this strengthened consistency test improves sample efficiency, and indeed the proofs go through (giving the same query complexity bounds), albeit with a significant weakening of the final guarantee: this version of \tensorplan could \emph{only} compete with optimal policies that are realizable, instead of the best of \emph{all} realizable DML policies.
\tensorplan, as presented, is able to compete with the latter, with its only source of pressure to do well coming from the optimistic choice of $\theta_\tau$ in \cref{line:new-theta} of {\tt Init}.
\end{remark}

The following theorem gives a query complexity guarantee on using \tensorplan to find a near-optimal policy. The precise values of $\zeta, \varepsilon, n_1, n_2,$ and $n_3$ mentioned in the theorem can be found in \cref{app:proofs}.
For the theorem statement recall that $\cBtheta$ is the bound on the $2$-norm of value-function parameter vectors that the algorithm competes with.
\begin{restatable}{theorem}{mainub}\label{thm:main}
\label{thm:ub}
For any $\delta>0$ and $B>0$,
there exists values of $\zeta, \varepsilon, n_1, n_2,$ and $n_3$ such that the
\tensorplan algorithm (\cref{alg:local}) is $(\delta,\cBtheta)$-sound (\cref{def:sound}) with misspecification $\eta=0$ and simulator accuracy $\lambda\le\epsilon/(4\sqrt{E_d})=\ordot\left(\left(\frac{\delta}{12\sqrt{d}H^2}\right)^A/\sqrt{A}\right)$ for the $H$-horizon planning problem with worst-case per-episode query-cost %
\[
\ordot\left(d^A A^4 \cBtheta^2/\delta^2 \left(H^5\cBtheta^2d/\delta^2+ d^AAH^{4(A+1)}12^{2A}/\delta^{2A}\right)\right)
= \poly\left(\left(dH/\delta\right)^A,B\right)
\,.
\]
\end{restatable}

\begin{corollary}
When the optimal value function $v^\star$ is linearly realizable with the given feature map with misspecification $\eta=0$, then \tensorplan, given access to a simulator with accuracy $\lambda\le\epsilon/(4\sqrt{E_d})$  induces a policy $\pi$ within the budget constraints of \cref{thm:main} for which $v^\pi_1(s_0)\ge v^\star_1(s_0)-\delta$.
\end{corollary}

\begin{proof}[Proof (of Theorem \ref{thm:main})]
We provide here a very brief sketch, and defer the full proof to Appendix \ref{app:proofs}.
The proof proceeds in a few steps. First, fix any starting state $s_0\in \cS$ and any $\theta^\circ\in\Theta^\circ$.
\cref{sec:concentration} establishes that
despite the simulator's inaccuracy,
the estimates $\hd$ and $\ld$ are close to their respective expected values  (\cref{lem:d-measured-accurately-in-any-dir})
and that $\ip{\ld,\theta^\circ}$ is close to its expected value (\cref{lem:d-measured-accurately-in-thetastar}).
This entails that $\theta^\circ$ does not get eliminated from the solution set (\cref{lem:theta-star-in-sol}). In \cref{sec:eluder-bound}, we use the eluder dimension to bound the maximal length of $X$ (essentially, the list of states where consistency is broken). It follows that, with high probability, the iteration over $\tau$ will be exited in \cref{line:cleantest-break} with $\tt CleanTest$ being true for $\tau \le E_d+1$. The last subsection (\cref{sec:value-bound}) bounds the suboptimality of the policy induced by $\thetafinal$ in terms of the inner product between $\thetafinal$ and the measured TD vectors (\cref{lem:value-accurate}). We then bound these suboptimalities by the desired suboptimality (\cref{cor:accurate-return}) and finally establish in \cref{cor:optimal-return} that the policy induced by the planner is $\delta$-optimal compared to $v_1(s_0;\theta^\circ)$. Since this argument holds for any $s_0\in\cS$ and $\theta^\circ\in\Theta^\circ$, the planner is $(\delta,\cBtheta)$-sound according to \cref{def:sound}.
\end{proof}

Our next theorem generalizes the previous results to the misspecified case (ie. $\eta>0$) by trading off simulator accuracy for misspecification. Formally, we provide a reduction to the realizable case and run \tensorplan with a slightly modified simulation oracle $\simulatesc'$ which requires no additional information beyond that provided by the original simulator. The proof is deferred to \cref{app:proof-misspecification}. The main idea of the proof is to define an alternate MDP with an expanded state space where states are indexed by which stage they belong to so that the misspecification error of a target policy can be ``pushed'' into the rewards of the new MDP. This way, the target policy will not have misspecification errors. The simulator for the new MDP still reports the rewards from the original MDP, but this is allowed since the previous result was stated for the case when the simulator introduces (small) errors when reporting the rewards.
\begin{restatable}{theorem}{thmmisspec}\label{thm:misspecification}
For any $\delta,B>0$,
\tensorplan is $(\delta,\cBtheta)$-sound with misspecification $\eta\le\epsilon/(12\sqrt{E_d})$ and simulator accuracy $\lambda\le\epsilon/(12\sqrt{E_d})$
with worst-case per-episode query-cost %
$
\poly\left(\left(dH/\delta\right)^A,B\right)
$, when run with input $\delta'=0.98\delta$ and simulation oracle $\simulatesc'$.
\end{restatable}

\subsection{Discounted MDPs}

In the discounted MDP setting, instead of maximizing the expected value of the reward $\sum_{h\in[H]} R_h$ over a horizon $H$,
the goal of the agent is to maximize the expected value of the
discounted total reward, $\sum_{h\in\N_+} \gamma^{h-1} R_h$,
over an infinite horizon, where $0\le \gamma<1$ is a fixed discount factor, given to the agent.
The value function for a policy $\pi$, $v^\pi:\cS\rightarrow \R$ is defined as $v^\pi(s)=r_\pi(s)+\gamma P_{sa}v^\pi$.
The stage index $h$ is dropped from the feature mapping ($\phi: \cS\rightarrow \R^d$), and the definition of v-linearly realizable policies (\cref{def:v-realizable-policies}) changes from requiring $\left|v^\pi_h(s) - \ip{\phi_h(s),\theta}\right|\le \eta$ to requiring
\[
\left|v^\pi(s) - \ip{\phi(s),\theta}\right|\le \eta \quad\quad \text{for all }s\in \cS\,.
\]
Soundness is otherwise defined identically to the $H$-horizon case, except for swapping the value function to $v^\pi$. Importantly, value guarantees are only required for the initial state the planner is called with, and not for every state that the planner ever encounters. 
As the episodes are infinitely long in this setting, we use the per-state (instead of per-episode) query-cost.

We use a reduction of the discounted case to the finite-horizon case with an ``effective horizon'' $\effectivehorizon$. 
Our next theorem shows that the guarantees of \tensorplan in the $\effectivehorizon$-horizon setting transfer to the discounted setting if it is run with a slightly modified simulation oracle $\simulatesc^{\gamma,\delta}$, which once again does not require any additional information beyond that of the original simulation oracle. 
As this is a reduction, the input $h$ given to \tensorplan's {\tt GetAction} should be incremented for each transition, exactly as in the finite-horizon case.
The definition of $\effectivehorizon$ and $\simulatesc^{\gamma,\delta}$, as well as the proof can be found in \cref{app:proof-discounted}.
\begin{restatable}{theorem}{thmdiscounted}\label{thm:discounted}
For any $\delta,B>0$,
\tensorplan is $(\delta,\cBtheta)$-sound
for discounted MDPs with discount factor $0\le\gamma<1$,
with misspecification $\eta\le\epsilon/(24\sqrt{E_d})$ and simulator accuracy $\lambda\le\epsilon/(12\sqrt{E_d})$,
with worst-case per-state query-cost %
$
\poly\left(\left(d\effectivehorizon/\delta\right)^A,B\right)
$, when run with input $\delta'=0.98\delta$ and simulation oracle $\simulatesc^{\gamma,\delta}$.
\end{restatable}

\section{Related work}\label{sec:related}

\paragraph{Planning with generative models}  
The local planning problem was introduced by \cite{kearns2002sparse}, who noticed that a planner which is given a simulator and an input state and asked to return a good action can do so with computation/query time independent of the size of the state space.
However, this runtime is exponential in $H$.
\citet{Munos2014-ty} gives algorithms that use optimism to improve on this exponential runtime
in benign cases.
With linear features, a negative result of \cite{Du_Kakade_Wang_Yan_2019} (see also \citet{van2019comments,LaSzeGe19}) states that an exponential in $\min\{H, d\}$ runtime remains for any planner with constant suboptimality, %
even if the feature map nearly realizes the action-value functions of \emph{all} policies
but the approximation error is $\varepsilon = \Omega(\sqrt{H/d})$.
For target suboptimality $\cO(\sqrt{d} \varepsilon)$,
assuming access to the solution of a feature-map-dependent optimal design problem,
\cite{LaSzeGe19} gives a planner with polynomial computational (and query) complexity.
These results are complemented by the lower bound of \cite{weisz2020exponential}, showing that an exponential lower bound still holds when only $q^\star$ is realizable even if there are no approximation errors.
When only the optimal value function is well-represented, \cite{Sharriff_Szepesvari_2020} give an algorithm for the case where the features are contained in the convex hull of a ``core set'' of feature vectors. Their planning algorithm, which builds on top of \cite{lakshminarayanan2017linearly}, has computational and query cost that scales polynomially in the size of the core set and the other relevant quantities. A similar approach appears in \cite{Zanette_Lazaric_Kochenderfer_Brunskill_2019}. By contrast, we only provide a bound only on the query complexity of our algorithm, but our query complexity is independent of the size of the core set, whose size, in general, is uncontrolled by the other quantities.

\paragraph{Online learning} %
Any online learning algorithm that controls regret can
 also be used for local planning by recommending the most frequently used action at the start state.
Of the sizable literature on online learning with linear function approximation
\citep{
Jiang_Krishnamurthy_Agarwal_Langford_Schapire_2017,
Du_Luo_Wang_Zhang_2019,
Jin_Yang_Wang_Jordan_2019,
wang2019optimism,
Yang_Wang_2019,
ayoub2020model,
modi2020sample,
wang2020provably,
Zanette_Lazaric_Kochenderfer_Brunskill_2020},
the most relevant are the works of \citet{Wen_Roy_2013,Jiang_Krishnamurthy_Agarwal_Langford_Schapire_2017}.
Both works give algorithms for the online setting with realizable function approximation, and are based on the principle of optimism. The algorithm of \citet{Wen_Roy_2013} is restricted to MDPs with deterministic rewards and deterministic transitions, and guarantees that at most $d$ trajectories will be suboptimal. Their proof is based on a similar eluder dimension argument. On the other hand, the algorithm of \citet{Jiang_Krishnamurthy_Agarwal_Langford_Schapire_2017} is restricted to the case when a complexity measure called the Bellman rank is low. In fact, our agnostic guarantee (see Definition~\ref{def:v-realizable-policies}) is related to a similar agnostic guarantee of their algorithm (see their Appendix A.2), where optimism at the initial state allows them to compete with the best policy whose state-value function is realizable. %
Despite the similarities, neither the algorithm nor the analysis applies to our setting.

\section{Conclusions and discussion}
\label{sec:conclusion}

We presented \tensorplan, a provably efficient algorithm for local planning in finite-horizon MDPs which only requires linear realizability of $v^\star$. When the action set is small (i.e. $\cO(1)$), \tensorplan
is the first algorithm that enjoys polynomial query complexity without further assumptions. Our results are also complemented by an exponential lower bound for the analogous problem in the infinite-horizon setting and an extension of the positive result to the near-realizable as well as the discounted setting.

In contrast to ADP-type algorithms \citep{SchSei85},
our algorithm does not use value fitting.
In fact, without stronger assumptions such as a core set, ADP algorithms appear to be susceptible to an exponential blow-up of errors \citep{tsitsiklis1996feature, dann2018oracle, Zanette_Lazaric_Kochenderfer_Brunskill_2019, wang2020statistical, weisz2020exponential}.
For the same reason, our algorithm works with a weaker simulation oracle that provides access only to states that have been encountered previously.
Learning via local consistency (``bootstrapping'') also allows us to provide a more agnostic guarantee, which automatically matches the best realizable value function. However, our lower bound suggests that such bootstrapping procedures are inefficient for the infinite-horizon setting, at least when the number of actions is larger (scaling with the feature space dimensionality). In offline RL, this issue was recently highlighted for the discounted setting in \cite{zanette2020exponential}.

There are several directions for future work. The first would be to understand the computational efficiency of our algorithm, or to find a computationally efficient alternative. The second would be to understand whether the exponential dependence of the query complexity on the number of actions is strictly necessary.
Lastly, it remains to be seen whether polynomial query complexity is possible under $q^\star$ realizability with a small number of actions.

\section*{Acknowledgements}

We thank the anonymous reviewers for their helpful comments. 
This work was done while the authors were visiting the Simons Institute for the Theory of Computing. PA gratefully acknowledges funding from the Natural Sciences and Engineering Research Council (NSERC). CS gratefully acknowledges funding from the Canada CIFAR AI Chairs Program, Amii and NSERC.

\bibliography{linear_fa}

\appendix

\section{Discussion of the Interaction Protocol}
\label{apx:inter}
The astute reader may wonder about whether the above description unduly restricts what state-spaces the planners can deal with: After all, the planner needs to be given ``states''.
This issue can be resolved by introducing state-identifiers: The planner and the simulator communicate by passing each other identifiers of states, rather than states themselves. The simulator then needs to be prepared to translate the identifiers to its internal state representation. This way, the planner can interact uniformly with MDPs of all sorts without any restriction on what their state-spaces are.

The reader may also wonder about whether access to states can be altogether avoided. In a way, as we shall see, our planner does not need the full power of the above interface.
In particular,
for planning within an episode starting at some initial state $s_0$,
it is sufficient if the oracle {\em (i)} has an internal state
and provides an interface to:
{\em (ii)} reset its internal state to the initial state $s_0$,
{\em (iii)} forward the internal state to a random next state by feeding it an action;
{\em (iv)} obtain data of the form $(f,X)$ where
		$X=(R_{a,i}+\Lambda_{sa},f_{a,i}')_{a\in \cA, i\in [n]}$
and \emph{where the value of $n$ is provided as an input},
	$f = \phi_h(s)$
and for $a\in\cA,i\in [n]$, $f'_{a,i} = \phi_{h+1}(S_{a,i}')$ where $(R_{a,i},S_{a,i}')_i \sim Q(s,a)$ and $\left|\Lambda_{sa}\right|\le\lambda$,
provided that the oracle's internal state is $S=s$ in stage $h$.

The reader may also be tempted to think that an even weaker interface that replaces {\em (iv)} with the ability to receive $f= \phi_h(s)$ provided that the oracle's internal state is $S=s$ in stage $h$ would be sufficient.
Clearly, this is too weak in the sense that even
the query complexity of the one-step lookahead calculation on the right-hand side of the Bellman optimality equation (\cref{eq:bopt}) would be uncontrolled.

\section{Proof of \cref{thm:lb}}
\label{apx:lbproof}
\begin{proof}
Fix the planner $p$ and $d>0$. For convenience, we will describe the case when the feature-space dimension is $d+1$.
We define a family of featurized MDPs: The MDP mentioned in the theorem statement will be a member of this family.
The MDPs in the family are all deterministic and they share the dynamics (and as such both a state space and action space).
The state space is a regular $d$-dimensional grid with $3^d$ points, say, $\cS = \{-1,0,1\}^d$.
The actions correspond to moving between neighboring states of the grid, or staying in put.
Thus, there are at most $2d+1$ actions: in each state: Incrementing, or decrementing a coordinate, or not changing anything.
For convenience, we will use costs in the description of the MDPs.
The MDPs in the family differ only in the costs assigned to transitions.
In the family there will be $2^d$ MDPs, each MDP defines the problem of getting to a goal state $s^*\in \{-1,1\}^d$ by taking the fewest number of actions.
This is achieved by setting the cost of each action to $1/d$ except that the cost of action `stay' provided that this action is taken in a ``goal state'' in which case the cost is set to zero.
We denote the resulting MDP by $M_{s^\star}$.
Fix $s^\star$. The optimal value function (negative cost) in $M_{s^\star}$ is $v^\star(s) = -\frac1d \norm{s-s^\star}_1$,
which indeed takes values in the $[-2,0]$ interval. %
Since for $x\in \{-1,0,1\}$, $x^\star \in \{-1,1\}$, $|x-x^\star| = 1- x x^\star$,
choosing $\phi(s) = \frac1d \overline{d s}$ we see that with
$\theta^\star = \overline{(-1) \, s^\star}$, $v^\star(s) = \ip{\phi(s),\theta^\star}$, i.e., the optimal value function is realizable under $\phi$. Note that the MDPs in the family not only share the dynamics, but they also share the feature map.

To find the MDP within this family on which planner $p$ is far from optimal we choose $s^\star$ in an adversarial manner.
For this, we define a new MDP,
$M_{\varnothing}$, which shares the dynamics with the previously described MDPs except that here all actions have a cost of $1/d$.
This MDP is used for ``testing'' the planner.
For an MDP $M$ either from the above family or $M_{\varnothing}$,
starting with $S_1=s_0:=(0,\dots,0)$, let
$S_1,Q_1,A_1,S_2,Q_2,A_2,\dots,S_n,Q_n,A_n, \dots$
be the infinite sequence of random elements that describe the data available to the planner while it
is used in $M$. In particular, $S_1$ is the state passed to the planner in the first call of {\tt GetAction},
$Q_1 \in \cup_{p\ge 0} (\cS \times \cA \times \bR \times \cS)^p$ collects the queries sent and the responses received, $A_1$ is the action return by this call of {\tt GetAction}, $S_2$ is the state the MDP transitions to from state $S_1$ on the effect of action $A_1$, etc.
Let $S_1',S_2',S_3',\dots$ collect the \emph{distinct} states in this data in the order that they are encountered.
In particular, $S_1' = S_1$ and if $Q_1 = (\tilde S_1, \tilde A_1, \tilde R_1, \tilde S_2', \dots,
\tilde S_{N_1}, \tilde A_{N_1}, \tilde R_{N_1}, \tilde S_{N_1}')$ with some $N_1\ge 0$ then
(since necessarily $\tilde S_1 = S_1$), $S_2'=\tilde S_2'$ unless $\tilde S_2' = S_1$, etc.

Let $S_{1:m}' = (S_1',\dots,S_m')$.
Slightly abusing notation, we treat this sequence as the set of the elements in it, when convenient.
\newcommand{\PP}{\mathbb{P}}
Let $\PP_{\varnothing}$ denote the probability distribution over the above data induced by interconnecting $p$ and $M_{\varnothing}$ and let $\PP_{s^\star}$ be the same when the MDP is $M_{s^\star}$.

An elementary argument shows that
$\min_{s\in \{-1,1\}^d} \PP_{\varnothing}( s\in S_{1:m} )\le m/2^d$.
Indeed,
$2^d \min_{s\in \{-1,1\}^d} \PP_{\varnothing}( s\in S_{1:m} )
    \le \sum_{s\in \{-1,1\}^d} \PP_{\varnothing}( s\in S_{1:m} )
    \le \sum_s \sum_{i=1}^m \PP_{\varnothing}(s=S_i)
= m$.
Let $s^\star$ be a minimizer of the above probability when $m=2^{d-1}$.
Consider the event $\cE = \{ s^{\star} \not\in S_{1:2^{d-1}} \}$.
By the above, $\PP_{\varnothing}(\cE) \ge 2^{d-1}/2^d = 0.5$.
Then $\PP_{s^\star}(\cE)=\PP_{\varnothing}(\cE)$.
This follows because for either of these probability measures,
$\PP(\cE) = \sum_{s^\star\not\in s_{1:m}} \PP( S_{1:m}'=s_{1:m} )$ and the probabilities in the sum
are identical for $\PP_{s^\star}$ and $\PP_{\varnothing}$ since the MDPs only differ in the cost assigned to the `stay put' action used \emph{at state $s^{\star}$} and the sum is restricted for state-sequences that avoid $s^\star$.

For each call of ${\tt GetAction}$, at most $\simb+1$ distinct states are encountered.
Therefore, ${\tt GetAction}$ is called at least $\lfloor 2^{d-1}/(\simb+1)\rfloor $ times during the time
the first $2^{d-1}$ distinct states are encountered.
On the event $\cE$, during this time, $s^\star$ is not encountered and thus the cost of each action executed is $1/d$.
Since the optimal cost to get to $s^\star$ from $s_0$ is one and all costs are nonnegative,
on $\cE$,
the total cost incurred while following the actions taken by the planner is at least $
\delta_0= \lfloor 2^{d-1}/(d(\simb+1)) \rfloor -1$ more than the optimal cost.
Since $\Probab_{s^\star}(\cE)\ge 0.5$, the suboptimality $\delta$ of the planner is at least $0.5 \delta_0$.
\end{proof}

\section{Proof of \cref{thm:ub}}
\label{app:proofs}

To prove that \tensorplan (\cref{alg:local}) is $(\delta,\cBtheta)$-sound (\cref{def:sound}) for the $H$-dimensional planning problem, we fix $\delta>0$, $B>0$, $H>1$,
a featurized MDP $(\sM,\phi)$ with $1$-bounded feature maps,
a suboptimality target $0<\delta< H$,
and a (starting) state state $s_0 \in \cS$.
We assume that $\delta<H$ as otherwise, for $\delta\ge H$, \cref{eq:soundness-satisfies-gap} trivially holds due to the rewards being bounded in $[0,1]$ (and therefore the values in $[0,H]$).

The precise values of hyperparameters used in \tensorplan will be set to:

\begin{align}
\zeta &= \frac{1}{4H}\delta \label{eq:algzetadef} \\
\epsilon &= \left(\frac{\delta}{12H^2}\right)^A/\left(1+\frac{1}{2\sqrt{E_d}}\right) \label{eq:algepsdef}  \\
n_1 &=  \ceil{\frac{32H^2(1+2\cBtheta)^2}{\delta^2} \log \frac{E_d+1}{\zeta}} \label{eq:algn1def} \\
n_2 &=  \ceil{\frac{1867H^2(\cBtheta+1)^2(d+1)}{2\delta^2}\log (4(E_d+1)n_1HA(d+1)/\zeta)} \label{eq:algn2def} \\
n_3 &= \ceil{\max\left\{n_2, \frac{32(H+1)^2E_d}{\epsilon^2}\log((2(E_d+1)n_1HA))/\zeta\right\}}
\label{eq:algn3def}
\end{align}
We assume $H>1$ for simplicity of presentation, as for $H=1$ the same analysis will apply, replacing $H$ with $H+1$ in the above display for $\epsilon$.

Denote by $\taufinal$ the final value of $\tau$ at the end of \tensorplan.{\tt Init}.
For the proof let $\mathbb{P}$ denote the probability distribution induced by the interconnection
of \tensorplan with the MDP when the initial state of the episode is $s_0$ and the planner is used for the $H$ steps. In particular, $\mathbb{P}$ is defined over
some measurable space $(\Omega,\mathbb{P})$ that carries the random variables
$S_1$, $A_1$, $S_2$, $A_2$, $\dots$, $A_H$, $S_{H+1}$,
where $S_1=s_0$, $S_i \sim P_{A_{i-1}}(S_{i-1})$ for $i>1$, and for $j\in[H]$, $A_j$ is the action returned by {\tt GetAction} when called with $S_j$ and $h=j$.
$(\Omega,\mathbb{P})$ also carries the random variables $\hat{\Delta}$, $\bar{\Delta}$, $\tilde{\Delta}$, and $(S_{\tau t j}, A_{\tau t j})_{\tau\le E_d+2, t\in [n_1],j\in [H]}$ of the \tensorplan algorithm.
For the latter, assume for now that \tensorplan.{\tt Init} does not break out from the loop over $\tau$ when the test fails, but that it keeps running, so that we can refer to $(S_{\tau t j},A_{\tau t j})$ even for $\tau>\taufinal$.
Note that all other quantities that appear in \tensorplan can be written as a function of these.
We denote the expectation operator underlying $\mathbb{P}$ by $\mathbb{E}$.

\subsection{Concentration bounds}\label{sec:concentration}

This section establishes concentration bounds on the estimated difference vectors $\hd$ and $\ld$, and then establishes that the true parameter is unlikely to be eliminated from the solution set.

\begin{restatable}{lemma}{hoeffding}\label{lem:d-measured-accurately-in-any-dir}
If the simulator's accuracy $\lambda\le\frac{\epsilon}{4\sqrt{E_d}}$, then
with $n_2$ samples for $\ld$ and $n_3$ samples for $\hd$, with probability greater than $1-\zeta$, for all $\theta\in\bR^d$ with $\norm{\theta}_2\le \cBtheta$, for all $\tau\in[E_d+1]$, $t\in[n_1]$, $j\in[H]$ and action $a\in[A]$, %
$\ld_{\tau tja}$ and $\hd_{\tau tja}$ satisfy %
\[
\left|\ip{\ld_{\tau tja}-\Delta(S_{\tau tj},a), \concat{1\theta}}\right| \le \delta/(12H)
\,\,\,\,\,\,\,\,\,\,\,\,\text{and}\,\,\,\,\,\,\,\,\,\,\,\,
\left|\ip{\hd_{\tau tja}-\Delta(S_{\tau tj},a), \concat{1\theta}}\right| \le \delta/(12H)\,,
\]
where $\Delta(S_{\tau tj},a) = \concat{ r_{S_{\tau tj}, a} (P_{S_{\tau tj}a}\phi_{j+1}-\phi_j(S_{\tau tj}))}$.
\end{restatable}
\begin{proof}
We show this for $\ld_{\tau tja}$, i,e., that the first inequality holds with probability at least $1-\zeta/2$. As $n_3\ge n_2$, by a similar argument this statement holds for $\hd_{\tau tja}$ too, and a union bound on the failure probability finishes the proof.
Let us refer here to the measurements $\tilde\Delta_l$ done by \approxmeasure called in \cref{line:avg-calc1} in \cref{alg:global} as $(\tilde\Delta_{\tau tjal})_{l\in[n_2]}$.
By the bounded rewards (the simulator's rewards are clipped in $[0,1]$ despite its inaccuracy), triangle inequality, and the assumption that $\forall h\in[H+1], s\in\cS,\, \norm{\phi_h(s)}_2\le 1$, we have that $\norm{\tilde\Delta_{\tau tjal}}_\infty\le \norm{\tilde\Delta_{\tau tjal}}_2\le 3$.
 
Since $\ld_{\tau tja}$ is the average of $n_2$ independent identically distributed bounded samples of the distribution of $\tilde\Delta_{\tau tjal}$, which has expectation $\Delta'(S_{\tau tj},a)=\concat{ \left(\left[r_{S_{\tau tj}, a}+\Lambda_{S_{\tau tj}, a}\right]_0^1\right) \left(P_{S_{\tau tj}a}\phi_{j+1}-\phi_j(S_{\tau tj})\right)}$, we can apply Hoeffding's inequality for each component $i\in[d+1]$ of the vector:
\[
\Probab\left(\left|\left(\ld_{\tau tja}-\Delta'(S_{\tau tj},a)\right)_i\right| > \delta/\left(\frac{72}{5}H(\cBtheta+1)\sqrt{d+1}\right)\right)
\le 2\exp\left( -\frac{2n_2 \delta^2}{\left(\frac{72}{5}\right)^2H^2(\cBtheta+1)^2(d+1)3^2} \right)
\]
Setting $n_2= \ceil{\frac{1867H^2(\cBtheta+1)^2(d+1)}{2\delta^2}\log (4(E_d+1)n_1HA(d+1)/\zeta)}$ allows this probability to be bounded by $\zeta/(2(E_d+1)n_1HA(d+1))$.
A union bound over $\tau\in[E_d+1]$, $t\in[n_1]$, $j\in[H]$, $a\in[A]$, and $i\in[d+1]$ achieves the $\zeta/2$ failure probability bound.
Under this high-probability event we have that $\norm{\ld_{\tau tja}-\Delta'(S_{\tau tj},a)}_\infty\le \delta/\left(\frac{72}{5}H(\cBtheta+1)\sqrt{d+1}\right)$, so $\left|\ip{\ld_{\tau tja}-\Delta'(S_{\tau tj},a), \concat{1\theta}}\right|\le \norm{\ld_{\tau tja}-\Delta'(S_{\tau tj},a)}_\infty \norm{\concat{1\theta}}_1\le \norm{\ld_{\tau tja}-\Delta'(S_{\tau tj},a)}_\infty \norm{\concat{1\theta}}_2\sqrt{d+1}\le \delta/(\frac{72}{5}H)$.
By the triangle inequality:
\[
\left|\ip{\ld_{\tau tja}-\Delta(S_{\tau tj},a), \concat{1\theta}}\right|
\le
\left|\ip{\ld_{\tau tja}-\Delta'(S_{\tau tj},a), \concat{1\theta}}\right| + \lambda
\le \delta/H\left(\frac{5}{72}+\frac{1}{72}\right)=\delta/(12H)\,,
\]
as $\lambda\le \frac{\epsilon}{4\sqrt{E_d}}\le \delta/(12H)/4/(1+\frac{1}{2})$.
\end{proof}

\begin{restatable}{lemma}{hoeffdingtheta}\label{lem:d-measured-accurately-in-thetastar}
If the simulator's accuracy $\lambda\le\frac{\epsilon}{4\sqrt{E_d}}$, then
with $n_3$ samples for $\hd$, with probability at least $1-\zeta$, for all $\tau\in[E_d+1]$, $t\in[n_1]$, $j\in[H]$ and action $a\in[A]$, %
\[
\left|\ip{\hd_{\tau tja}-\Delta(S_{\tau tj},a), \concat{1\theta^\circ}}\right| \le \frac{\epsilon}{2\sqrt{E_d}}
\]
where $\Delta(S_{\tau tj},a) = \concat{ r_{S_{\tau tj}, a} (P_{S_{\tau tj}a}\phi_{j+1}-\phi_j(S_{\tau tj}))}$.
\end{restatable}
\begin{proof}
Let us refer here to the measurements $\tilde\Delta_l$ done by \approxmeasure called in \cref{line:avg-calc2} in \cref{alg:global} as $(\tilde\Delta_{\tau tjal})_{l\in[n_3]}$.
Since $\theta^\circ\in\Theta^\circ$, $\theta^\circ$ satisfies \cref{eq:theta-circ-satisfies} for some policy.
Furthermore, due to the bounded rewards, horizon $H$, and the simulator's clipping of rewards into $[0,1]$ (despite its inaccuracy), and the bounded values (of any state for any policy) in $[0,H]$, we have that $\ip{\tilde\Delta_{\tau tjal}, \concat{1\theta^\circ}} \in[-(H+1),(H+1)]$.
Since $\hd_{\tau tja}$ is the average of $n_3$ independent identically distributed bounded samples of the distribution of $\tilde\Delta_{\tau tjal}$, which has expectation $\Delta'(S_{\tau tj},a)=\concat{ \left(\left[r_{S_{\tau tj}, a}+\Lambda_{S_{\tau tj}, a}\right]_0^1\right) \left(P_{S_{\tau tj}a}\phi_{j+1}-\phi_j(S_{\tau tj})\right)}$, we can apply Hoeffding's inequality:
\[
\Probab\left(\left|\ip{\hd_{\tau tja}, \concat{1\theta^\circ}} - \ip{\Delta'(S_{\tau tj},a), \concat{1\theta^\circ}}\right|> \frac{\epsilon}{4\sqrt{E_d}}\right)
\le 2\exp\left( -\frac{n_3 \epsilon^2}{32(H+1)^2E_d} \right)\,.
\]
Setting $n_3=\ceil{\max\left\{n_2, \frac{32(H+1)^2E_d}{\epsilon^2}\log((2(E_d+1)n_1HA))/\zeta\right\}}$ allows this probability to be bounded by $\zeta/((E_d+1)n_1HA)$.
By the triangle inequality, under the high-probability event, the desired bound with $\Delta$ instead of $\Delta'$ is guaranteed as:
\[
\left|\ip{\hd_{\tau tja}, \concat{1\theta^\circ}} - \ip{\Delta(S_{\tau tj},a), \concat{1\theta^\circ}}\right| \le
\left|\ip{\hd_{\tau tja}, \concat{1\theta^\circ}} - \ip{\Delta'(S_{\tau tj},a), \concat{1\theta^\circ}}\right| + |\Lambda_{S_{\tau tj}, a}|
 \le 2\frac{\epsilon}{4\sqrt{E_d}}
\]
A union bound over $\tau\in[E_d+1]$, $t\in[n_1]$, $j\in[H]$, and $a\in[A]$ achieves the desired probability bound.
\end{proof}

\begin{restatable}[$\theta^\circ \in \sol(X)$]{lemma}{thetasol}\label{lem:theta-star-in-sol}
For $\tau \in [E_d+1]$, let $X_{\le\tau}$ denote the first $\tau$ elements of $X$, where $X$ is defined in \cref{line:x-def} of \cref{alg:global}. Then, with probability at least $1-\zeta$ we have that
$\forall \tau\in[E_d+1]$,
$\theta^\circ \in \sol(X_{\le\tau})$. %
\end{restatable}
\begin{proof}
As in \cref{lem:d-measured-accurately-in-thetastar}, by MDP reward boundedness, $\left|\ip{\Delta(S_{\tau tj},a), \concat{1\theta^\circ}}\right| \le H$ for any %
$\Delta(S_{\tau tj},a)$. %
Let $A^\circ_{\tau tj}$ be the action satisfying \cref{eq:loccons} for $\theta^\circ$ in state $S_{\tau tj}$.
Then we have that
$\ip{\Delta(S_{\tau tj},A^\circ_{\tau tj}), \concat{1\theta^\circ}}=0$.
Thus, using \cref{lem:d-measured-accurately-in-thetastar}, with probability at least $1-\zeta$, for all $\tau\in[E_d+1]$, $t\in[n_1]$, $j\in[H]$, $a\in[A]$,
\begin{align*}
\left|\ip{\hd_{\tau tja}, \concat{1\theta^\circ}}\right|
&=   \left|\ip{\Delta(S_{\tau tj},a), \concat{1\theta^\circ}} + \ip{\hd_{\tau tja}-\Delta(S_{\tau tj},a), \concat{1\theta^\circ}}\right|\\
&\le \left|\ip{\Delta(S_{\tau tj},a), \concat{1\theta^\circ}}\right| + \left|\ip{\hd_{\tau tja}-\Delta(S_{\tau tj},a), \concat{1\theta^\circ}}\right|\\
&\le \one{a\ne A^\circ_{\tau tj}}H + \frac{\epsilon}{2\sqrt{E_d}} \,,
\end{align*}
where $\one{S}$ is the indicator of a set $S$. We can then bound the product across $a\in[A]$ as
\[
\prod_{a\in[A]} \ip{\hd_{\tau tja}, \concat{1\theta^\circ}} \le \left(H+\frac{\epsilon}{2\sqrt{E_d}}\right)^{A-1}\frac{\epsilon}{2\sqrt{E_d}}=
\left(1+\frac{\epsilon}{2\sqrt{E_d}H}\right)^{A-1}H^{A-1}\frac{\epsilon}{2\sqrt{E_d}}\,,
\]
and
\begin{align*}
\left(1+\frac{\epsilon}{2\sqrt{E_d}H}\right)^{A-1} &\le 1+(2^{A-1}-1)\frac{\epsilon}{2\sqrt{E_d}H}
<1+2^A\epsilon\\
&<1 + 2^A\frac{\delta^A}{(12H^2)^{A}}=1+\left(\frac{2\delta}{12H^2}\right)^A<2\le H\,,
\end{align*}
so $\prod_{a\in[A]} \ip{\hd_{\tau tja}, \concat{1\theta^\circ}} < H^A\frac{\epsilon}{2\sqrt{E_d}}$.
Let $\tau\in[E_d+1]$. The $\tau^\text{th}$ element added to $X$ will be $\otimes_{a \in [A]}\hd_{\tau tja}$ computed in \cref{line:avg-calc2} of \cref{alg:global} for some $\tau\in[E_d+1]$, $t\in[n_1]$, $j\in [H]$, so $\theta^\circ \in \sol(X_{\le\tau})$ according to \cref{eq:sol}.
\end{proof}

\subsection{Eluder dimension}\label{sec:eluder-bound}

This subsection uses the eluder dimension to bound the maximal number of iterations. For $\Theta\in\bR^{(d+1)^A}$ and $x\in \bR^{(d+1)^A} $, let
\[
f_\Theta(x) = \ip{\text{clip}(x),\Theta},
\]
where $\text{clip}(x)=\frac{x}{\norm{x}_2} \min\{\norm{x}_2, 3^A\}$.
Notice the similarity between these functions and the form of the constraints we use in \cref{eq:sol} to define the set of parameter vectors $\sol(\cdot)$ consistent with our observations.
Let
\[
\cF^+ = \{f_\Theta \,:\, \Theta\in\bR^{(d+1)^A},\,\norm{\Theta}_2\le (\cBtheta+1)^A\}
\]
and
\begin{align}
E_d = \floor{3(d+1)^A\frac{e}{e-1} \ln\left\{ 3+3\left(\frac{2(\cBtheta+1)^A 3^A}{H^A\epsilon}\right)^2 \right\}+1} = \ordot\left(d^A A\right) \,.
\label{eq:eddef}
\end{align}
By \citet{russo2014learning},
$\dime(\cF^+, H^A\epsilon)$, the \textbf{eluder dimension} of $\cF^+$ at scale $H^A\epsilon$ is the length $\tau$ of the longest \textbf{eluder sequence} $x_1,\ldots,x_\tau$, such that for some $\epsilon'\ge H^A\epsilon$, for each $l\in [\tau]$,
\[
w_l := \sup\left\{ \left|f_1(x_l)-f_2(x_l)\right| \,:\, \sqrt{\sum_{i=1}^{l-1}(f_1(x_i)-f_2(x_i))^2}\le \epsilon',\,\,f_1,f_2\in\cF^+ \right\} > \epsilon' \,.
\]
Also by \citet{russo2014learning} (Appendix C.2), $\dime(\cF^+, H^A\epsilon)\le E_d$.
Now let
\begin{align}
\cF = \{f_\Theta \,:\, \exists \theta\in\bR^d .\, \norm{\theta}_2\le \cBtheta,\, \Theta=\text{flatten}(\otimes_{a\in[A]}\concat{1\theta})\}\,.
\label{eq:tensorizedclass}
\end{align}
Since $\norm{\theta}_2\le \cBtheta$ implies $\norm{\text{flatten}(\otimes_{a\in[A]}\concat{1\theta})}_2\le (\cBtheta+1)^A$, $\cF\subseteq \cF^+$, and so $\dime(\cF, H^A\epsilon)\le \dime(\cF^+, H^A\epsilon)\le E_d$.

\begin{restatable}{lemma}{eluderseq}\label{lem:new-eluder-element}
With probability at least $1-2\zeta$,
at any point in the execution of \cref{alg:global}, the sequence $X_{\le E_d+1}$ %
is an eluder sequence for $\cF$ at scale $H^A\epsilon$.
\end{restatable}

\begin{proof}
Let us assume the event under which $\theta^\circ\in\sol(X_{\le\tau})$ for $\tau\in[E_d+1]$, which has probability at least $1-\zeta$ by \cref{lem:theta-star-in-sol}.
Let us also assume the high-probability event of \cref{lem:d-measured-accurately-in-any-dir}.
Let $\epsilon'=H^A\epsilon$.
The empty sequence is trivially an eluder sequence. By induction, assume for some $\tau\in[E_d+1]$ that $X_{\le \tau-1}$ is an eluder sequence. %
Let $\flattheta^\circ=\text{flatten}(\otimes_{a\in[A]}\concat{1\theta^\circ})$ and
let $\flattheta_j=\text{flatten}(\otimes_{a\in[A]}\concat{1\theta_j})$.
\begin{align*}
w_\tau &= \sup\left\{ \left|f_1(X_\tau)-f_2(X_\tau)\right| \,:\, \sqrt{\sum_{i=1}^{\tau-1}(f_1(X_i)-f_2(X_i))^2}\le H^A\epsilon,\,\,f_1,f_2\in\cF \right\}\\
&\ge    \sup\left\{ \left|f_1(X_\tau)-f_2(X_\tau)\right| \,:\,
\forall i\in[\tau-1] \,.\, \left|(f_1(X_i)-f_2(X_i))\right|\le \frac{H^A\epsilon}{\sqrt{E_d}},\,\,f_1,f_2\in\cF \right\}\\
&\ge     \left|f_{\flattheta_\tau}(X_\tau)-f_{\flattheta^\circ}(X_\tau)\right| > \left(\delta/(4H)\right)^A- \left|f_{\flattheta^\circ}(X_\tau)\right|  > H^A\epsilon\left(1+\frac{1}{2\sqrt{E_d}}\right) - \frac{H^A\epsilon}{2\sqrt{E_d}} = H^A\epsilon \,,
\end{align*}
where the first line expands the definition of $w_\tau$, the second comes from proving that $\forall i\in[\tau-1] \,.\, \left|(f_1(X_i)-f_2(X_i))\right|\le \frac{H^A\epsilon}{\sqrt{E_d}}$ implies $\sqrt{\sum_{i=1}^{\tau-1}(f_1(X_i)-f_2(X_i))^2}\le H^A\epsilon$.
We show this by assuming the former and letting $v\in\bR^{\tau-1}$ be $v_i=f_1(X_i)-f_2(X_i)$, and then $\norm{v}_2\le \norm{v}_\infty\sqrt{\tau-1}\le H^A\epsilon$ as $\tau-1\le E_d$ by the induction assumption.

The last line comes from substituting $f_1=f_{\flattheta_\tau}$ and $f_2=f_{\flattheta^\circ}$.
For this we have to show that $f_{\flattheta_\tau}, f_{\flattheta^\circ}\in\cF$, and that
$\forall i\in[\tau-1] ,\, \left|(f_{\flattheta_\tau}(X_i)-f_{\flattheta^\circ}(X_i))\right|\le \frac{H^A\epsilon}{\sqrt{E_d}}$. The former holds by definition (as $\norm{\theta^\circ}_2\le \cBtheta$ and $\norm{\theta_\tau}_2\le \cBtheta$ as $\theta_\tau\in\sol(X_{\le \tau-1})$).
For the latter, we use that $\theta^\circ,\theta_\tau\in\sol(X_{\le \tau-1})$,
so for either $\flattheta\in\{\flattheta^\circ,\flattheta_\tau\}$, $\forall i\in[\tau-1] \,.\, \left|f_{\flattheta}(X_i)\right| \le \left|\ip{X_i, \flattheta}\right|\le \frac{H^A\epsilon}{2\sqrt{E_d}}$, so by the triangle inequality,
$\forall i\in[\tau-1] \,.\, \left|f_{\flattheta_\tau}(X_i)-f_{\flattheta^\circ}(X_i)\right| \le  \frac{H^A\epsilon}{\sqrt{E_d}}$.
Finally, it is left to show that $\left|f_{\flattheta_\tau}(X_\tau)-f_{\flattheta^\circ}(X_\tau)\right| > H^A\epsilon$. For some $t\in[n_1]$, $j\in[H]$, $X_{\tau}=\otimes_{a\in[A]} \hd_{\tau tja}$.
Since $\norm{\hd_{\tau tja}}_2\le 3$, 
$\norm{X_\tau}_2\le 3^A$, so
$\forall f_\flattheta\in\cF$, $f_\flattheta(X_\tau)=\ip{\text{clip}(X_\tau),\flattheta}=\ip{X_\tau,\flattheta}$.
Furthermore, because the algorithm added $(X_{\tau a})_a=(\hd_{\tau tja})_a$ in \cref{line:new-eluder-element}, %
$\min_{a\in[A]} \left|\ip{\ld_{\tau tja},\concat{1\theta}_\tau}\right|> \delta/(4H)=3H\epsilon^{1/A}\left(1+\frac{1}{2\sqrt{E_d}}\right)^{1/A}$.
Under the assumed high-probability event of \cref{lem:d-measured-accurately-in-any-dir}, for $a\in[A]$, since $\tau\in[E_d+1]$ and $t\in[n_1]$, by \cref{lem:d-measured-accurately-in-any-dir} and the triangle inequality,
$\left|\ip{\ld_{\tau tja},\concat{1\theta}_\tau}\right| - \left|\ip{\hd_{\tau tja},\concat{1\theta}_\tau}\right| \le 2\delta/(12H)$, so
$\min_{a\in[A]} \left|\ip{\hd_{\tau tja},\concat{1\theta}_\tau}\right|> \delta/(12H)=H\epsilon^{1/A}\left(1+\frac{1}{2\sqrt{E_d}}\right)^{1/A}$, therefore
$\prod_{a\in[A]} \left|\ip{\hd_{\tau tja},\concat{1\theta}_\tau}\right|> H^A\epsilon\left(1+\frac{1}{2\sqrt{E_d}}\right)$.
We finish by bounding $\left|f_{\flattheta^\circ}(X_\tau)\right|\le \frac{H^A\epsilon}{2\sqrt{E_d}}$ as $\theta^\circ\in\sol(X_{\le \tau})$ by our high-probability assumption, so by the triangle inequality, and noting that $f_{\flattheta_\tau}(X_\tau)=\prod_{a\in[A]} \ip{\hd_{\tau tja},\concat{1\theta}_\tau}$, we have that
$\left|f_{\flattheta_\tau}(X_\tau)-f_{\flattheta^\circ}(X_\tau)\right| \ge \prod_{a\in[A]} \left|\ip{\hd_{\tau tja},\concat{1\theta}_\tau}\right| - \left|f_{\flattheta^\circ}(X_\tau)\right| > H^A\epsilon\left(1+\frac{1}{2\sqrt{E_d}}\right)-\frac{H^A\epsilon}{2\sqrt{E_d}}$.
\end{proof}

By definition of the eluder dimension, we then have:
\begin{restatable}{corollary}{iterationbound}\label{cor:final-tau-bound}
With probability at least $1-2\zeta$, $\taufinal\le\dime(\cF, H^A\epsilon)+1\le E_d+1$.
\end{restatable}
\begin{proof}
Assume the high-probability statements of \cref{lem:new-eluder-element} hold and that $\taufinal>\dime(\cF, H^A\epsilon)+1$. Take $X_{\le\dime(\cF, H^A\epsilon)+1}$ which is of length $\dime(\cF, H^A\epsilon)+1$. Also, $\dime(\cF, H^A\epsilon)+1\le E_d+1$. Therefore, by \cref{lem:new-eluder-element}, $X_{\le\dime(\cF, H^A\epsilon)+1}$ is an eluder sequence for $\cF$ at scale $H^A\epsilon$ of length $>\dime(\cF, H^A\epsilon)$, which is a contradiction.
\end{proof}

\subsection{Value bound}\label{sec:value-bound}

Denote by $\tpp$ the policy induced by \tensorplan. %

\begin{restatable}{lemma}{valueaccurate}\label{lem:value-accurate}
With probability $1-2\zeta$, if $\taufinal\in[E_d+1]$,
$v_1^{\tpp}(s_0)\ge \ip{\phi_1(s_0), \thetafinal} - \frac{1}{n_1}\sum_{t\in[n_1]}
	\sum_{j\in[H]} \ip{\ld_{\taufinal tjA_{\taufinal tj}}, \concat{1\thetafinal}} - \frac{1}{2}\delta$.
\end{restatable}

\begin{proof}
Let us denote the state we reach after $H$ steps (once the episode is over) by $S_{H+1}$ in the following. For $\psi:\cS \to \R^d$, we let $P_{sa}\psi = \int \psi(s') P_{sa}(ds')$. 

Recall that under $\mathbb{P}$, the random variables
$S_1=s_0$, $A_1$, $S_2$, $A_2$, $\dots$, $A_H$, $S_{H+1}$
have the distribution of an episode in the MDP that starts from $s_0$ and follows the policy $\tpp$ induced by \tensorplan.
\begin{align*}
v_1^{\tpp}(s_0)
&= \E \sum_{j\in[H]} r_{S_j,A_j}
= \E \ip{\concat{\left(\sum_{j\in[H]} r_{S_j,A_j}\right)\phi_{H+1}(S_{H+1})}, \concat{1\thetafinal}}\\
&= \E \left[\ip{\phi_1(s_0),\thetafinal} + \sum_{j\in[H]} \ip{\concat{r_{S_j,A_j}(\phi_{j+1}(S_{j+1})-\phi_j(S_j))}, \concat{1\thetafinal}}\right]\\
&= \ip{\phi_1(s_0),\thetafinal} + \sum_{j\in[H]} \E\ip{\concat{r_{S_j,A_j}(\phi_{j+1}(S_{j+1})-\phi_j(S_j))}, \concat{1\thetafinal}}\\
&= \ip{\phi_1(s_0),\thetafinal} + \sum_{j\in[H]}
	\E
	\left[
		\ip{\concat{r_{S_j,A_j} ( P_{S_j A_j} \phi_{j+1}-\phi_j(S_j))}, \concat{1\thetafinal}}\right]\\
&\ge \ip{\phi_1(s_0),\thetafinal} + \frac{1}{n_1}\sum_{t\in[n_1]}
	\sum_{j\in[H]} %
		 \left[\ip{\concat{r_{S_{\taufinal tj},A_{\taufinal tj}}( P_{S_{\taufinal tj}A_{\taufinal tj}} \phi_{j+1}-\phi_j(S_{\taufinal tj}))}, \concat{1\thetafinal}} \right] - \frac{1}{4}\delta \\
&\ge \ip{\phi_1(s_0),\thetafinal} + \frac{1}{n_1}\sum_{t\in[n_1]}
	\sum_{j\in[H]} \ip{\ld_{\taufinal tjA_{\taufinal tj}}, \concat{1\thetafinal}} - \frac{1}{2}\delta \,,
\end{align*}
where in the first line we used that $\phi_{H+1}(S_{H+1})=\boldsymbol{0}$, %
in the second that $s_0=S_1$, in the third that $s_0$ is fixed so can be moved out of the expectation, and in the fourth we used the tower rule for expectations.
In the fifth line we replace the outer expectation with an average of rollouts by the algorithm that is close to the expectation with high probability, while we also switched to the variable notation used in \cref{alg:global}.
More specifically, we use the fact that
for all $h\in[H+1]$, $s\in\cS$, and $\tau\in[E_d+1]$,
we have that $\norm{\phi_h(s)}_2\le 1$ and
$\norm{\thetafinal}_2\le \cBtheta$, $\left| \ip{\concat{r_{S_{\tau tj},A_{\tau tj}}( P_{S_{\tau tj}A_{\tau tj}} \phi_{j+1}-\phi_j(S_{\tau tj}))}, \concat{1\thetafinal}}\right| \le 1+2\cBtheta $ (as rewards are bounded in $[0,1]$).
We can therefore apply Hoeffding's inequality on the $n_1$ independent rollouts:
\begin{align*}
&\Probab\Bigg( \frac{1}{n_1}\sum_{t\in[n_1]}\Bigg[
	\sum_{j\in[H]}\Big[ \ip{\concat{r_{S_{\tau tj},A_{\tau tj}}( P_{S_{\tau tj}A_{\tau tj}} \phi_{j+1}-\phi_j(S_{\tau tj}))}, \concat{1\thetafinal}} \\
&\quad\quad- \E\ip{\concat{r_{S_{\tau tj},A_{\tau tj}}( P_{S_{\tau tj}A_{\tau tj}} \phi_{j+1}-\phi_j(S_{\tau tj}))}, \concat{1\thetafinal}} \Big]\Bigg]
   > \delta/4\Bigg)\\
&\,\,\,\,\le \exp\left( -\frac{n_1 \delta^2}{32H^2(1+2\cBtheta)^2} \right) \le \frac{\zeta}{E_d+1}\,,
\end{align*}
if $n_1= \ceil{32H^2(1+2\cBtheta)^2/\delta^2 \log \frac{E_d+1}{\zeta}}$.
With an union bound, the probability that any of these bounds fail for any $\tau\in[E_d+1]$ is upper bounded by $\zeta$.
We can therefore apply this bound for $\tau=\taufinal$, noting that %
\[
\E\ip{\concat{r_{S_{\taufinal tj},A_{\taufinal tj}}( P_{S_{\taufinal tj}A_{\taufinal tj}} \phi_{j+1}-\phi_j(S_{\taufinal tj}))}, \concat{1\thetafinal}}
=
\E\ip{\concat{r_{S_j,A_j} ( P_{S_j A_j} \phi_{j+1}-\phi_j(S_j))}, \concat{1\thetafinal}}\,.
\]
This is because $\thetafinal=\theta_{\taufinal}$, so for all $t\in[n_1]$, the episode
$\left(S_{\taufinal t1},A_{\taufinal t1},\dots, A_{\taufinal tH}, S_{\taufinal t,H+1}\right)$
is distributed identically to the episode $(S_1, A_1, S_2, A_2, \dots, A_H, S_{H+1})$.
Finally, in the sixth line we replace the remaining expectation with the average measured by the algorithm, which is close to the expectation with high probability (\cref{lem:d-measured-accurately-in-any-dir}) for $\tau\in[E_d+1], t\in[n_1], j\in[H], a\in[A]$.
By a union bound, this adds another $\zeta$ to the probability that our bound does not hold.
\end{proof}

\begin{corollary}\label{cor:accurate-return}
With probability at least $1-3\zeta$, %
$v_1^{\tpp}(s_0)\ge \ip{\phi_1(s_0), \thetafinal} - \frac{3}{4}\delta$.
\end{corollary}
\begin{proof}
Under the high-probability event of \cref{cor:final-tau-bound}, $\taufinal\le E_d+1$. %
From the proof of \cref{lem:value-accurate}:
\begin{align*}
v_1^{\tpp}(s_0) &\ge \ip{\phi_1(s_0),\thetafinal} + \frac{1}{n_1}\sum_{t\in[n_1]}
	\sum_{j\in[H]} \ip{\ld_{\tau tjA_{\tau tj}}, \concat{1\thetafinal}} - \frac{1}{2}\delta
 \ge \ip{\phi_1(s_0),\thetafinal} - \frac{3}{4}\delta \,
\end{align*}
where we use the fact that, since $\taufinal\le E_d+1$, we exited the $\tau$ loop as {\tt CleanTest} was true in \cref{line:cleantest-break}, so for $\taufinal$, all $t\in[n_1]$ the path in \cref{line:action-choice} was chosen (otherwise we would have finished with a larger $\taufinal$). This directly bounds the inner product of interest. Taking a union bound over the underlying high-probability events, $v_1^{\tpp}(s_0)\ge \ip{\phi_1(s_0), \thetafinal} - \frac{3}{4}\delta$ holds with probability at least $1-3\zeta$.
\end{proof}
\begin{corollary}\label{cor:optimal-return}
$v_1^{\tpp}(s_0)\ge v_1(s_0;\theta^\circ) - \delta$.
\end{corollary}
\begin{proof}
Assume all high-probability events introduced so far, which hold with probability at least $1-3\zeta$.
By \cref{cor:final-tau-bound}, $\taufinal\le E_d+1$.
By \cref{lem:theta-star-in-sol}, $\theta^\circ\in\sol(X_{\le \taufinal})$.
Since $\thetafinal$ was chosen optimistically in \cref{line:new-theta}, $\ip{\phi_1(s_0), \thetafinal} \ge \ip{\phi_1(s_0), \theta^\circ}=v_1(s_0;\theta^\circ)$.
By \cref{cor:accurate-return}, $v_1^{\tpp}(s_0)\ge \ip{\phi_1(s_0), \thetafinal} - \frac{3}{4}\delta\ge v_1(s_0;\theta^\circ)-\frac{3}{4}\delta$.
Therefore, with probability at least $1-3\zeta=1-\frac{1}{4H}\delta$, $v_1^{\tpp}(s_0)\ge v_1(s_0;\theta^\circ) - \frac{3}{4}\delta$, so $v_1^{\tpp}(s_0)\ge \left(1-\frac{1}{4H}\delta\right)\left(v_1(s_0;\theta^\circ) - \frac{3}{4}\delta\right)\ge v_1(s_0;\theta^\circ) - \delta$ (using that due to bounded rewards, $v_1^{\tpp}(s_0)\le H$).

\end{proof}

\subsection{Final bound}
We can now combine all the ingredients together to get the final result.

\mainub*
\begin{proof}
Fix $\delta>0$ and $B>0$. 
By \cref{cor:optimal-return}, $v_1^{\tpp}(s_0)\ge v_1(s_0;\theta^\circ) - \delta$ for any $\theta^\circ\in\Theta^\circ$.
Denoting by $v^\circ=v_B^\circ$ the H-horizon $\phi$-compatible optimal value function of $\sM$, $v_1^{\tpp}(s_0)\ge \sup_{\theta^\circ\in\Theta^\circ} v_1(s_0;\theta^\circ) - \delta = v_1^\circ(s_0)-\delta$ by definition, proving soundness.
In each episode, $\cref{line2:avg-calc}$ in \tensorplan.{\tt GetAction} is called $H$ times, and \tensorplan.{\tt Init} is called once. The former results in $Hn_2A$ calls to the simulator. We turn our attention to the query complexity of \tensorplan.{\tt Init}.
The loop variable $\tau$ of {\tt Init} goes up to $E_d+2$ so $\taufinal\le E_d+2$.
\cref{line:avg-calc1} can therefore be called at most $(E_d+2) n_1 H A$ times, each performing $n_2$ interactions with the simulator.
\cref{line:simulate-choice} can be called at most $(E_d+2) n_1 H$ times, each performing $1$ interaction with the simulator.
\cref{line:avg-calc2} can be called at most $(E_d+2) n_1 A$ times, each performing $n_3$ interactions with the simulator.
Using that $E_d=\ordot\left(d^A A\right)$, $\lambda=\ordot\left(\left(\frac{\delta}{12\sqrt{d}H^2}\right)^A/\sqrt{A}\right)$.
Furthermore, using that
$n_1=\ordot\left(H^2\cBtheta^2A/\delta^2\right)$,
$n_2=\ordot\left(H^2\cBtheta^2dA/\delta^2\right)$,
$n_3=\ordot\left(d^A A^2 H^2/\epsilon^2+H^2\cBtheta^2dA/\delta^2\right)= \ordot\left(d^AA^2H^{4A+2}12^{2A}/\delta^{2A} +H^2\cBtheta^2dA/\delta^2\right)$,
the (worst-case per-episode) query-cost of \tensorplan (along any episode) is
\begin{align*}
\ordot\left(Hn_2A+E_d n_1 A \left(Hn_2+n_3\right)\right)
&=\ordot\left(E_d n_1 A \left(Hn_2+n_3\right)\right)
=\ordot\left(d^A A^3 H^2\cBtheta^2/\delta^2 \left(Hn_2+n_3\right)\right)\\
&=\ordot\left(d^A A^4 \cBtheta^2/\delta^2 \left(H^5\cBtheta^2d/\delta^2+ d^AAH^{4(A+1)}12^{2A}/\delta^{2A}\right)\right)\,. \tag*{\qedhere}
\end{align*}
\end{proof}

\section{Proof of \cref{thm:misspecification}}
\label{app:proof-misspecification}

\thmmisspec*
\begin{proof}
Fix $\delta>0$, $H\ge1$, $\eta=\epsilon/(12\sqrt{E_d})$ and $\lambda=\epsilon/(12\sqrt{E_d})$.
We assume that $\delta<H$ as soundness trivially holds otherwise.
Let $(\sM,\phi)$ be any featurized MDP with $1$-bounded feature maps and rewards bounded in $[0,1]$.
Let $\simulatesc$ be the $\lambda$-accurate simulation oracle for $(\sM,\phi)$.
We will shortly define a slightly modified simulation oracle $\simulatesc'$ corresponding to a featurized MDP $(\sM',\phi')$ derived from $(\sM,\phi)$. This oracle will simply use the data returned from calls to $\simulatesc$ while we will claim that it is a 
simulator for $(\sM',\phi')$
with inaccuracy not more than $\epsilon/(4\sqrt{E_d})$.

Denote by $\tpp$ the policy while \tensorplan interacts with the simulator $\simulatesc'$.
By the correspondence between the two MDPs, $\tpp$ can be interpreted as a policy of $\sM$.
We will then prove that for all states $s\in \cS$ of $\sM$, 
\begin{align*}
v^{\tpp}_1(s) \ge v^\circ_1(s)-\delta\,,
\end{align*}
where $v^{\tpp}$ is the $H$-horizon value function of \tensorplan's policy $\tpp$ in $\sM$ and $v^\circ=v^\circ_{B,\eta}$ is the $H$-horizon $\phi$-compatible optimal value function of $\sM$ (cf. Equation \eqref{eq:phi-compatible-value}).

Let $\Pi^\circ_{B,\eta}$ be the set of memoryless, deterministic (MLD) policies that are $B$-boundedly
$v$-linearly realizable with misspecification $\eta$ and features $\phi$.
Then, by definition, $v^\circ_{B,\eta}(s)=\sup_{\pi\in\Pi^\circ_{B,\eta}} v^\pi_1(s)$.
It is enough to prove that for any $\pi\in\Pi^\circ_{B,\eta}$,
\begin{align*}
v^{\tpp}_1(s) \ge v^{\pi}_1(s)-\delta\,.
\end{align*}
Fix a $\theta\in\R^d$ such that $\left|v^\pi_h(s) - \ip{\phi_h(s),\theta}\right|\le \eta$ for all $h\in [H]$ and $s\in \cS$. Such a $\theta$ exists by definition.
We now construct an alternative featurized MDP $(\sM', \phi')$
that will mimic $\sM$, but with slightly different rewards and an expanded state-space.
The main point of introducing this MDP is that the value function of $\pi$ (when ``used'' in $\sM'$) will be realizable with $\eta=0$. The function $\simulatesc'$ will be defined to act as a simulator for $(\sM',\phi')$.
Then we will use an extension \cref{thm:main} to argue that \tensorplan induces a policy that can compete with $\pi$ in $\sM'$ and hence, by the correspondence between the two MDPs, it also competes with $\pi$ in $\sM$.
The required extension of \cref{thm:main} is as follows:
\begin{claim}\label{claim:mainp}
The conclusions of \cref{thm:main} remain valid with the following two changes:
\begin{enumerate}[(i)]
\item The rewards in the MDP are allowed to belong to $[-2,2]$;
\item %
A set $\cS_1\subset \cS$ is fixed and the requirement of soundness is redefined so that the initial state chosen at the beginning of an episode must belong to $\cS_1$ while
 v-realizability (cf. \cref{def:v-realizable-policies}) of a policy $\pi$ 
is redefined 
so that instead of
$\max_{h\in [H]} \sup_{s\in \cS} \left|v^\pi_h(s) - \ip{\phi_h(s),\theta}\right|\le \eta$ we require 
$\max_{h\in [H]} \sup_{s\in \cS_h} \left|v^\pi_h(s) - \ip{\phi_h(s),\theta}\right|\le \eta$ where $\cS_h\subset \cS$ is defined as the set of states that can be reached
with positive probability in $\sM$ from some state in $\cS_1$ and action sequence of length $h-1$. 
\end{enumerate}
\end{claim} 
\begin{proof}
For \emph{(i)} note that shifting the rewards does not impact the proof, while the range of rewards scales the query cost quadratically (this comes from the use of Hoeffding's inequality, where ranges of temporal difference errors appear, which scale linearly with the range of rewards).
For \emph{(ii)} we only need to check that if $\theta^\circ$ is a parameter vector of a policy with the modified definition, this parameter vector will not be eliminated by \tensorplan. 
A quick look at the proof of \cref{lem:theta-star-in-sol} confirms that this is the case.
Indeed, \tensorplan constructs data for checking consistency at stage $h$ only with states that it reaches through $h-1$, or $h$ actions from the initial state it is given. Therefore the states that appear with $\phi_h$ always belong to $\cS_h$. As such, \cref{lem:theta-star-in-sol}  continues to hold true, and the result follows.
\end{proof}

Let us now return to the definition of $\sM'=(\cS',\Sigma', [A],Q')$ and $\phi'$.
We let the states of $\sM'$ be $\cS'=\cS \times [H]\cup \{\bot\}$, that is, the state space of $\sM'$ contains $H$ copies of each state, and a final absorbing state $\bot$. The intention is that only states of the form $(s,h+1)$ are accessible from states of the form $(s,h)$.
We let $\Sigma'$ to be the smallest $\sigma$ algebra under which $\{\bot\}$ and all the sets of the form $S\times \{h\}$ are measurable where $S\in \Sigma$ and $h\in [H]$.
We let $\phi'_h((s,\cdot))=\phi_h(s)$ and $\phi'_h(\bot)=\bm{0}$, a $d$-dimensional vector of all zeros.

The transition kernel $Q'$ in $\sM'$ will follow that in $\sM$, with the appropriate modification to create the promised ``levelled'' structure, while the rewards are modified to ``cancel out the misspecification'' of policy $\pi$.
That is, for $h<H$,
 from state $(s,h)\in\cS'$ taking 
 action $a\in \cA$,
 kernel $Q'$ gives $(R+z(s,h),(S',h+1))$ where $(R,S') \sim Q_{sa}$ and
\[
z(s,h)=\E_{a'\sim\pi^{(h)}(s)}\left[ \ip{\phi_{h}(s)-P_{sa'}\phi_{h+1}, \theta}-r_{sa'} \right]\,.
\]
From state $(s,H)\in\cS'$ or $\bot$, any action leads deterministically to $\bot$ while incurring zero reward.

Notice that any $(s',h)\in\cS'$ can only be reached after exactly $h$ steps when starting from some other state $(s,1)$, $s \in \cS$. 
Furthermore, denoting by $r'$ the immediate rewards in $\sM'$, we have
$r'_{(s,h),a}=r_{sa}+z(s,h)$. 
Note that $|z(s,h)|\le 2\eta$, since $\left|v^\pi_h(s) - \ip{\phi_h(s),\theta}\right|\le \eta$ for all $h\in [H]$ and $s\in \cS$,
and $\E_{a'\sim\pi_h(s)}\left[ v^\pi_h(s)-P_{sa'}v^\pi_{h+1}-r_{sa'} \right]=0$ by the Bellman equation.
Hence, the rewards in $\sM'$ are supported on $[-2\eta,1+2\eta]\subset [-2,2]$ (as $\eta< 1/2$). 

For any $(s,h)\in\cS'$, let $\bar{v}'_{h}((s,h))=\ip{\phi'_{h}((s,h)), \theta}=\ip{\phi_{h}(s), \theta}$.
We claim that $\bar{v}'_{h}$ satisfies the Bellman equation of $\pi$ 
when policy $\pi$ in $\sM'$ is taken as a policy of $\sM'$ with the understanding that in state $(s,h)$ and stage $h$, following $\pi$ means using $\pi_h(s)$, while in stage $h'\ne h$, an arbitrary action can be taken. 
Indeed,  for any $(s,h)\in \cS'$ we have
\begin{align*}
 \bar{v}'_{h}((s,h)) = \ip{\phi_{h}(s),\theta} &= E_{a\sim\pi^{(h)}(s)}\left[r_{sa}+ \ip{\phi_{h}(s)-P_{sa}\phi_{h+1}, \theta}-r_{sa}  +\ip{P_{sa}\phi_{h+1},\theta}\right] \\
 &=E_{a\sim\pi^{(h)}(s)}\left[r_{sa}+ E_{a'\sim\pi^{(h)}(s)}\left[\ip{\phi_{h}(s)-P_{sa'}\phi_{h+1}, \theta}-r_{sa'}\right]  +\ip{P_{sa}\phi_{h+1},\theta}\right] \\
 &=E_{a\sim\pi^{(h)}(s)}\left[r'_{(s,h),a}+\ip{P_{sa}\phi_{h+1},\theta}\right] \\
 &= E_{a\sim\pi^{(h)}(s)}\left[r'_{(s,h),a}+P'_{(s,h),a}\bar{v}'_{h+1}\right] \\
 &= r'_\pi((s,h))+P'_\pi((s,h))\bar{v}'_{h+1}\,,
\end{align*}
where $P'$ is the transition kernel in $\sM'$ and $P'_\pi$ ($r'_\pi$)
is the corresponding kernel (respectively, reward function) induced by $\pi$.
Since $v'^\pi$ also satisfies this equation and $\bar{v}'_{H+1}=v'^\pi_{H+1}=\boldsymbol{0}$, it follows that for any $(s,h)\in \cS'$,
$v'^\pi_{h}((s,h))=\bar{v}'_{h}((s,h))=\ip{\phi'_{h}((s,h)), \theta}$.
Now, define $\cS_1' = \cS \times \{1 \}$. Then, $\cS_h'$, the set of states reachable in $\sM'$ with positive probability from $\cS'_1$ with an action sequence of length $h-1$, is easily seen to be a subset of $\cS \times \{ h \}$. Therefore, policy $\pi$ is v-realizable with $\eta'=0$ in the sense of the definition of v-realizability given in Part~\emph{(ii)} of \cref{claim:mainp}.

For state and action $s\in\cS, a\in[A]$,
recall that $\simulatesc(s,h,a)$ is implemented by a $\lambda$-accurate simulator for $(\sM, \phi)$, and that
the state transitions of $\sM$ and $\sM'$ are the same apart from that in the latter the stage counter is incremented in each transition. 
Hence, we define  $\simulatesc'$ as follows: $\simulatesc'((s,h),h',a)$ for $(s,h)\in\cS'$ calls $(R,S',\phi_{h'+1}(S'))\gets \simulatesc(s,h',a)$ and returns $(R, (S',h+1), \phi_{h'+1}(S'))$ for $h<H$ and $(R,\bot,\bm{0})$ otherwise. We also let $\simulatesc'(\bot,\cdot,\cdot)$ deterministically returns $(0,\bot,\bm{0})$.

Let $\pi'$ be a policy of $\sM'$ that is induced by a planner interacting with $\sM'$ using $\simulatesc'$ where the episode starts in $\sM'$ are restricted to $\cS_1'$. Then, on the one hand, $\pi'$ can be seen as a policy in $\sM$: For a history in $\sM$, one just needs to add the respective stage counters to the states in the history and then use $\pi'$ to return an action. 

Now note that the reward distribution of $\sM'$ is shifted by up to $2\eta$ compared to the reward distribution of $\sM$.
The distribution of the simulator's rewards $\left[R_{sa}+\Lambda_{sa}\right]_0^1$ are shifted by up to $\Lambda_{sa}\le\lambda$ compared to the reward distribution of $\sM$, so by the triangle inequality it is shifted by up to $2\eta+\lambda$ compared to the reward distribution of $\sM'$.
Since $2\eta+\lambda=\epsilon/(4\sqrt{E_d})$, using the reward of the simulator call $\simulatesc(s,h',a)$ as the output of $\simulatesc'((s,h),h',a)$ ensures $\simulatesc'$ is a simulator for $(\sM',\phi')$ with inaccuracy $\epsilon/(4\sqrt{E_d})$.

Therefore, applying \cref{claim:mainp} with $\eta'=0$, $\lambda'=\epsilon/(4\sqrt{E_d})$, and $\delta'=0.98\delta$, \tensorplan is $(\delta',\cBtheta)$-sound for $\sM'$ and initial states from $\cS'$ when run with the simulator $\simulatesc'$, with worst-case per-episode query-cost $\poly\left(\left(dH/\delta\right)^A,B\right)$.
Thus, for all $(s,1)\in\cS'$ (ie. all $s\in\cS$), $v_1'^{\tpp}((s,1))\ge v'^\circ_1((s,1))-0.98\delta$, where $v'^\circ=v'^\circ_{B,0}$ is the $H$-horizon $\phi$-compatible optimal value function of $\sM'$.
As $\pi$ is $v$-linearly realizable in MDP $\sM'$ with no misspecification, $v'^\circ_1((s,1))\ge v'^\pi_1((s,1))$, so $v_1'^{\tpp}((s,1))\ge v'^\pi_1((s,1))-0.98\delta$.
As the state transition distributions of $\sM$ and $\sM'$ are the same except for the stage counter incrementation in $\sM'$, the distribution of any policy $\pi$ in $\sM$ producing an episode $(S_1, A_1,S_2,A_2\dots,S_H,A_H)$ is the same as the distribution of $\pi$ in $\sM'$ producing the episode $((S_1,1), A_1,(S_2,2),A_2,\dots,(S_H,H),A_H)$.
Furthermore, the rewards of $\sM'$ are shifted by up to $2\eta$. Therefore, the $H$-horizon value functions $v_1^{\pi'}(s)$ and $v_1'^{\pi'}((s,1))$ for any $\pi'$ differ by at most $2H\eta$, and thus by treating $\tpp$ as a policy of both $\sM$ and $\sM'$, we have
\[
v_1^{\tpp}(s)\ge v'^\pi_1((s,1))-0.98\delta-2H\eta \ge v^\pi_1(s)-0.98\delta-4H\eta\ge v^\pi_1(s)-\delta\,,
\]
as $4H\eta=\frac{H\epsilon}{3\sqrt{E_d}}\le \frac{H\frac{\delta}{12H^2}/(1+0.5)}{3\sqrt{E_d}}\le \frac{\delta}{18H}/3\le 0.02\delta$.
\end{proof}

We note in passing that the result as stated could be (slightly) strengthened and simplified: 
Since $\simulatesc'$ generates the same data (with some redundancy) as $\simulatesc$,
using \tensorplan on $(\sM', \phi')$ via $\simulatesc'$ 
produces the same policy in $\sM$ as using it directly 
on $(\sM,\phi)$ via $\simulatesc$.
Thus, $\simulatesc'$ is only needed for the proof; the conclusion of the result applies when \tensorplan   directly uses $\simulatesc$ with a near-realizable featurized MDP.

By reiterating the arguments of 
\cref{claim:mainp} in the context of \cref{thm:misspecification}, we get the following claim, which will be needed in the next section:
\begin{claim}\label{claim:mainp2}
The conclusions of \cref{thm:misspecification} remain valid with the following two changes:
\begin{enumerate}[(i)]
\item The rewards in the MDP are allowed to belong to $[-2,2]$;
\item A set $\cS_1\subset \cS$ is fixed and the requirement of soundness is redefined so that the initial state chosen at the beginning of an episode must belong to $\cS_1$ while
 v-realizability (cf. \cref{def:v-realizable-policies}) of a policy $\pi$ 
is redefined 
so that instead of
$\max_{h\in [H]} \sup_{s\in \cS} \left|v^\pi_h(s) - \ip{\phi_h(s),\theta}\right|\le \eta$ we require 
$\max_{h\in [H]} \sup_{s\in \cS_h} \left|v^\pi_h(s) - \ip{\phi_h(s),\theta}\right|\le \eta$ where $\cS_h\subset \cS$ is defined as the set of states that can be reached
with positive probability in $\sM$ from some state in $\cS_1$ and action sequence of length $h-1$. 
\end{enumerate}
\end{claim}

\section{Proof of \cref{thm:discounted}}
\label{app:proof-discounted}

\thmdiscounted*
\begin{proof}
Fix a suboptimality target $\delta>0$.
We assume that $\delta<H$ as soundness trivially holds otherwise.
Fix $\eta=\epsilon/(24\sqrt{E_d})$ and $\lambda=\epsilon/(12\sqrt{E_d})$; proving soundness and the query-cost bound for these values implies the same results for smaller $\eta$ or $\lambda$.
Let $(\sM,\phi)$ be a featurized MDP in the discounted setting with 1-bounded feature maps and rewards bounded in $[0,1]$.
Take a $\lambda$-accurate simulation oracle $\simulatesc$ for $(\sM,\phi)$.
Let
\[
\effectivehorizon=\ceil{\frac{\log\left((1-\gamma)\eta\right)/\log\gamma}{1-\gamma}}\,.
\]
In the remainder of the proof we shorten $\effectivehorizon$ and will just use $\effectivehorizonshort$ (i.e., in what follows $\effectivehorizonshort=\effectivehorizon$).
We now
construct a featurized, fixed-horizon MDP $(\sM',\phi^{\gamma,\delta})$ with horizon $\effectivehorizonshort$.
Let the states of $\sM'$ be $\cS'=\cS \times [H]\cup \{\bot\}$, that is, the state space contains $\effectivehorizonshort$ copies of each state, and an additional state $\bot$, which will play the role of a final, absorbing state.
The $\sigma$ algebra for $\cS'$ is constructed as in the proof of \cref{thm:misspecification} (we omit the definition). 
The action set of $\sM'$ remains $[A]$.
The kernel $Q'$ is inherited from $\sM$, again, with the appropriate modification to create the promised ``levelled'' structure, while the rewards are modified to accommodate discounting:
That is, for $h<\effectivehorizonshort$,
 from state $(s,h)\in\cS'$ taking 
 action $a\in \cA$,
 kernel $Q'$ gives $(\gamma^{h-1} R,(S',h+1))$ where $(R,S') \sim Q_{sa}$.
From state $(s,\effectivehorizonshort)\in\cS'$ or $\bot$, any action leads deterministically to $\bot$ while incurring zero reward.
In words, states with associated stage $h<\effectivehorizonshort$ lead to respective states with associated stage $h+1$, and the episode is terminated after $\effectivehorizonshort$ steps by transitioning to the absorbing state $\bot$.
By letting $r'$ denote the immediate expected rewards in $\sM'$,
 for state $(s,h)\in\cS'$ and action $a$  we have $r'_{(s,h),a}=\gamma^{h-1}r_{sa}$. 

Let $\phi^{\gamma,\delta}_h((s,\cdot))=\gamma^{h-1}\phi(s)$ and $\phi^{\gamma,\delta}_h(\bot)=\bm{0}$, a $d$-dimensional vector of all zeros.
We define 
$\simulatesc^{\gamma,\delta}$ as follows:
$\simulatesc^{\gamma,\delta}$ is a simulation oracle 
for $(\sM',\phi^{\gamma,\delta})$ so that
 for $(s,h)\in\cS'$
 with $h<\effectivehorizonshort$,
 $h'\in [\effectivehorizonshort]$ and $a\in [A]$,
$\simulatesc^{\gamma,\delta}((s,h),h',a)$ first gets 
$(R,S,\phi(S))\gets \simulatesc(s,h',a)$ to return 
$\left(\gamma^{h-1}R,(S,h+1),\phi^{\gamma,\delta}_{h'+1}((S,h+1))\right)$,
while it returns $\left(\gamma^{\effectivehorizonshort-1}R,\bot,\bm{0}\right)$ when $h=\effectivehorizonshort$.
Finally,
$\simulatesc^{\gamma,\delta}(\bot,\cdot,\cdot)$ deterministically returns $(0,\bot,\bm{0})$.
As $\gamma <1$, the inaccuracy of $\simulatesc^{\gamma,\delta}$ is at most the inaccuracy of $\simulatesc$, which is at most $\lambda$, by assumption.

Next, we prove that the value function of the discounted MDP $\sM$ is close to the corresponding values of its $\effectivehorizonshort$-horizon counterpart $\sM'$. 
For this, we first need to agree on a way of transporting policy between $\sM$ and $\sM'$.
This is done as follows: Let $\alpha$ be a function that maps histories in $\sM$ to histories in $\sM'$ by adding stage counters to them. Let $\alpha^{-1}$ be the ``inverse'', which simply drops stage indices from histories of $\sM'$. For any $h$ history of $\sM$, $\alpha^{-1}(\alpha(h))=h$,
while $\alpha(\alpha^{-1}(h'))=h'$ holds for all histories $h'$ of $\sM'$ whose start state is from
 $\cS_1' = \cS \times \{ 1\}$ and where the states in the history do not include $\bot$.
If $\pi'$ is any (possibly memoryful) policy of $\sM'$, following $\pi'$ in $\sM$ means that given some history $h$ of $\sM$, the action $A\sim \pi'(\cdot|\alpha(h))$ should be taken.
Conversely, 
using a policy $\pi$ of $\sM$ in $\sM'$ means that given a history $h'$, $A \sim \pi(\cdot|\alpha^{-1}(h'))$ should be taken. This way, we can view a policy of either $\sM$ or $\sM'$ as a policy of the other MDP.

Now take any policy $\pi$ of $\sM$ and take any $(s,h)\in\cS'$. As $\pi$ is also a policy of $\sM'$, we can talk about its value function in $\sM'$, which we denote by $v'^\pi$. 
By definition, $v'^\pi_h( (s,h_0) ) = \E'_{\pi,(s,h_0)}[ \sum_{h'=1}^{\effectivehorizonshort-h+1} r'_{(S_{h'},h_0+h'-1),A_{h'}}]$, where $\E'_{\pi,s'}$ denotes the expectation operator underlying the distribution $\bbP'_{\pi,s'}$ over state-action trajectories
induced by the interconnection of $\pi$ and $\sM'$ given the initial state $s'\in \cS'$.
Similarly, we will use $\E_{\pi,s}$ to denote this operator when the MDP is $\sM$ and the initial state is $s\in \cS$, and we let $\bbP_{\pi,s}$ denote the underlying distribution.
With this note that 
\begin{align}
\bbP'_{\pi,(s,h)}(U\times V) = \bbP_{\pi,s}(\alpha(U\times V))
\label{eq:meid}
\end{align}
holds for any measurable subset $U$ of $(\cS\times [H]\times [A])^{H-h+1}$ and where $V = (\cS\times [H]\times [A])^{\mathbb{N}_+}$ is the set of all histories.
We claim that the following holds:
\begin{align}\label{ineq:discounted-values-close}
\left|v'^\pi_h\left((s,h)\right) -\gamma^{h-1}v^\pi(s)\right| \le \eta\,.
\end{align}
We calculate
\begin{align*}
\MoveEqLeft
\left|v'^\pi_h\left((s,h)\right) -\gamma^{h-1}v^\pi(s)\right| \\
=&\,\,
\left| 
\E'_{\pi,(s,h)}\left[\sum_{h'=1}^{\effectivehorizonshort-h+1} r'_{(S_{h'},h+h'-1),A_{h'}}\right] 
  - \gamma^{h-1}\E_{\pi,s}\left[\sum_{h'=1}^\infty \gamma^{h'-1} r_{S_{h'},A_{h'}}\right]\right|\\
=&\,\,
\Bigg| \gamma^{h-1} \E'_{\pi,(s,h)}\left[\sum_{h'=1}^{\effectivehorizonshort-h+1} \gamma^{h'-1}r_{S_{h'},A_{h'}}\right] -  \gamma^{h-1}
	\E_{\pi,s}\left[\sum_{h'=1}^{\effectivehorizonshort-h+1} \gamma^{h'-1} r_{S_{h'},A_{h'}}\right] \\
&\quad-\gamma^{h-1}
\E_{\pi,s}\left[\sum_{h'=\effectivehorizonshort-h+2}^{\infty} \gamma^{h'-1} r_{S_{h'},A_{h'}}\right]
\Bigg|\\
=&\,\,
\left| - \gamma^{\effectivehorizonshort}\E_{\pi,s}\left[\sum_{h'=\effectivehorizonshort-h+2}^\infty \gamma^{h'-(\effectivehorizonshort-h+2)} r_{S_{h'},A_{h'}}\right]\right|
\tag{by \cref{eq:meid}}
\\
&\le
\gamma^{\effectivehorizonshort}
\sum_{i=0}^\infty \gamma^i=\frac{\gamma^{\effectivehorizonshort}}{1-\gamma}\le \frac{\gamma^{\log((1-\gamma)\eta)/\log\gamma}}{1-\gamma}=\eta\,,
\end{align*}
where in the last line used the fact that rewards are bounded in $[0,1]$.
Now, notice that if $\pi$ was a policy of $\sM'$, \cref{eq:meid} would still hold true, and as such, 
\cref{ineq:discounted-values-close} also holds for $\pi$.

Take any policy $\pi$ that is $v$-linearly realizable in $\sM$ with misspecification $\eta$ under the feature map $\phi$. By definition, there exists a $\theta\in\R^d$ such that $\left|v^\pi(s)-\ip{\phi(s),\theta}\right|\le \eta$ for all $s\in\cS$ (ie. for all $(s,h)\in\cS'$).
By \cref{ineq:discounted-values-close} and the triangle inequality, for all $(s,h)\in\cS'$,
\begin{align*}
\left|v'^\pi_h((s,h))-\ip{\phi^{\gamma,\delta}_h((s,h)),\theta}\right|
&=
\left|v'^\pi_h((s,h))-\gamma^{h-1}\ip{\phi(s),\theta}\right| \\
&\le \left|v'^\pi_h((s,h))-\gamma^{h-1}v^\pi(s)\right|  + \gamma^{h-1}\left|v^\pi(s)-\ip{\phi(s),\theta}\right|\le 2\eta\,.
\end{align*}
Therefore any such policy $\pi$ is $v$-linearly realizable in MDP $\sM'$ with misspecification $2\eta$ under the feature map $\phi^{\gamma,\delta}$ for the respective stage $h$ for each state $(s,h)\in\cS'$. %

Therefore we can apply 
\cref{claim:mainp2}
for featurized MDP $(\sM', \phi^{\gamma,\delta})$,
initial set $\cS\times \{1\}$,
 and $\lambda$-accurate simulator $\simulatesc^{\gamma,\delta}$, with misspecification $\eta'=2\eta$ and $\delta'=0.98\delta$, which guarantees that \tensorplan is $(\delta',\cBtheta)$-sound for MDP $\sM'$ when run with this simulator and features.
Furthermore, it has a worst-case per-state query-cost $\poly\left(\left(d\effectivehorizonshort/\delta\right)^A,B\right)$.
Denote by $\tpp$ the policy induced by \tensorplan while interacting with the simulator $\simulatesc^{\gamma,\delta}$. We then have that $\tpp$ satisfies
\begin{align*}
v'^{\tpp}_1((s,1)) \ge v'^\circ_1((s,1))-0.98\delta\,,
\end{align*}
where $v'^\tpp$ is the $\effectivehorizonshort$-horizon value function of $\tpp$ in $\sM'$ and $v'^\circ=v'^\circ_{B,2\eta}$ is the  $\effectivehorizonshort$-horizon $\phi^{\gamma,\delta}$-compatible optimal value function of $\sM'$ under misspecification $2\eta$ (cf. Equation \eqref{eq:phi-compatible-value}).
Similarly, let $v^\circ=v^\circ_{B,\eta}$ be the  discounted $\phi$-compatible optimal value function of $\sM$ under misspecification $\eta$.
Let $\Pi'^\circ_{B,2\eta}$ be the set of MLD policies that are $B$-boundedly
$v$-linearly realizable in MDP $\sM'$ with misspecification $2\eta$ and features $\phi^{\gamma,\delta}$,
and let $\Pi^\circ_{B,\eta}$ be the set of MLD policies that are $B$-boundedly
$v$-linearly realizable in $\sM$ with misspecification $\eta$ and features $\phi$.
Then, by definition, $v'^\circ_{B,2\eta}(s)=\sup_{\pi\in\Pi'^\circ_{B,2\eta}} v'^\pi_1((s,1))$
and $v^\circ_{B,\eta}(s)=\sup_{\pi\in\pi'^\circ_{B,\eta}} v^\pi(s)$.

As we have seen, $\pi\in \Pi^\circ_{B,\eta}$ implies $\pi\in \Pi'^\circ_{B,2\eta}$, in other words, $\Pi^\circ_{B,\eta}\subseteq \Pi'^\circ_{B,2\eta}$.
For any policy $\pi$ \cref{ineq:discounted-values-close} applies with any $(s,1)\in\sM'$, and therefore
\begin{align*}
v^\circ_{B,\eta}(s)&=\sup_{\pi\in\Pi^\circ_{B,\eta}} v^\pi(s)
\le
\sup_{\pi\in\Pi'^\circ_{B,2\eta}} v^\pi(s)
\le
\sup_{\pi\in\Pi'^\circ_{B,2\eta}} v'^\pi_1\left((s,1)\right)+\eta\\
&=
v'^\circ_{B,2\eta}((s,1))+\eta
\le
v'^{\tpp}_1((s,1))+0.98\delta+\eta
\le
v'^{\tpp}(s)+0.98\delta+2\eta\,,
\end{align*}
where the last inequality used again \cref{ineq:discounted-values-close} with $\tpp$.
Lastly we use that $\eta=\frac{\epsilon}{24\sqrt{E_d}}\le \frac{\frac{\delta}{12\effectivehorizonshort^2}/(1+0.5)}{24\sqrt{E_d}}\le \frac{\delta}{18\effectivehorizonshort}/24< 0.01\delta$ to obtain $v^\circ_{B,\eta}(s) \le
v'^{\tpp}_1((s,1))+\delta$, which establishes that \tensorplan's policy $\tpp$ is $(\delta,\cBtheta)$-sound for the featurized MDP $(\sM,\phi)$ in the discounted setting, with misspecification $\eta \leq \frac{\epsilon}{24\sqrt{E_d}}$ and simulator accuracy $\lambda \leq \frac{\epsilon}{12\sqrt{E_d}}$. %
\end{proof}

\end{document}